\documentclass{article}

\usepackage[preprint]{neurips_2026}

\usepackage{hyperref}
\usepackage[utf8]{inputenc}
\usepackage[T1]{fontenc}
\usepackage{microtype}
\usepackage{enumitem}

\usepackage{amsmath,amssymb,amsfonts,amsthm}
\usepackage{mathtools}
\usepackage{bm}

\usepackage{multirow}

\usepackage{tikz}
\usepackage{pgfplots}
\usepgfplotslibrary{groupplots}
\pgfplotsset{compat=1.18} 

\usepackage{graphicx}
\usepackage{booktabs}
\usepackage{nicefrac}
\usepackage{algorithm}
\usepackage{algpseudocode}
\usepackage{tabularx}

\usepackage{hyperref}
\usepackage{url}
\usepackage{natbib}

\usepackage[table]{xcolor}
\usepackage{colortbl}
\usepackage[most]{tcolorbox}
\usepackage{amsmath,amssymb,amsthm}

\newcommand{\KL}{\operatorname{KL}}

\usepackage{pifont}

\definecolor{darkblue}{RGB}{0,51,102}
\definecolor{alertred}{RGB}{180,0,0}
\definecolor{successgreen}{RGB}{0,110,0}
\definecolor{lightyellow}{RGB}{255,252,220}
 
\hypersetup{
  colorlinks=true,
  linkcolor=blue!60!black,
  citecolor=teal!60!black,
  urlcolor=magenta!60!black
}

\definecolor{headerblue}{RGB}{220,230,242}
\definecolor{lightgrayrow}{RGB}{245,245,245}
\definecolor{poegreen}{RGB}{228,242,228}

\definecolor{headerorange}{RGB}{252,235,205}
\definecolor{shiftmethod}{RGB}{255,245,230}

\definecolor{headerpurple}{RGB}{234,226,245}
\definecolor{alphahighlight}{RGB}{245,240,250}

\definecolor{headerslate}{RGB}{224,228,235}
\definecolor{ablationhighlight}{RGB}{242,244,247}

\definecolor{headerbeige}{RGB}{245,236,220}
\definecolor{ablationsoft}{RGB}{252,247,240}

\definecolor{headerteal}{RGB}{221,238,236}
\definecolor{ablationteal}{RGB}{240,248,247}

\definecolor{headerrose}{RGB}{244,228,236}
\definecolor{priorrose}{RGB}{252,243,247}

\definecolor{appblue}{RGB}{226,236,248}
\definecolor{appgreen}{RGB}{232,244,234}
\definecolor{appyellow}{RGB}{250,243,214}
\definecolor{apppink}{RGB}{246,231,238}
\definecolor{apppurple}{RGB}{236,229,247}
\definecolor{appteal}{RGB}{227,242,240}

\definecolor{boxborder}{RGB}{120,120,120}
\definecolor{boxbg}{RGB}{248,248,248}

\newtheorem{theorem}{Theorem}
\newtheorem{lemma}{Lemma}
\newtheorem{proposition}{Proposition}
\newtheorem{corollary}{Corollary}
\theoremstyle{definition}

\newcommand{\TV}{\mathrm{TV}}

\newcommand{\supp}{\mathrm{supp}}

\title{When Policies Cannot Be Retrained: A Unified Closed-Form View of Post-Training Steering in Offline Reinforcement Learning}

\author{%
  Elias Hossain$^{1}$,
  Mohammad Jahid Ibna Basher$^{1}$,
  Ivan Garibay$^{1}$,
  Ozlem Garibay$^{1}$,
  Niloofar Yousefi$^{1}$ \\
  $^{1}$College of Engineering and Computer Science\\
  University of Central Florida, Orlando, FL, USA \\\\
  \texttt{\{mdelias.hossain, niloofar.yousefi\}@ucf.edu}
}

\begin{document}

\maketitle

\begin{abstract}
Offline RL can learn effective policies from fixed datasets, but deployment objectives often change after training, and in many applications the trained actor cannot be retrained (restricted data access, costly retraining, or governance requirements). We study deployment-time adaptation for frozen offline actors through a Product-of-Experts (PoE) composition with a goal-conditioned prior, and our central practical finding is one of graceful degradation rather than universal performance gain: under degraded or random priors, precision-weighted composition anchors to the frozen actor while additive and prior-only adaptation collapse, and a KL-budget selector typically lands at a near-oracle operating point. As a clarifying unification in this deployment setting we make explicit a closed-form identity that the frozen-actor adaptation literature has left implicit: for diagonal-Gaussian actors and priors, PoE with coefficient $\alpha$ is the same deterministic policy as the KL-regularized update with $\beta=\alpha/(1-\alpha)$, the posterior covariances differing only by a global scalar $1+\beta$. As an empirical characterization of where this actor-anchored composition is worth applying, four D4RL environments ($3{,}900$ MuJoCo episodes) give an illustrative $4$ / $5$ / $3$ \textsc{Help}/\textsc{Frozen}/\textsc{Hurt} split; a full $\alpha\in\{0.1,0.3,0.5,0.7,0.9\}$ sweep on six harder cells (medium-replay, medium-expert; $5{,}850$ episodes) and two AntMaze diagnostics then expose an \emph{actor-competence ceiling}: medium-expert remains \textsc{Hurt} in all $9$ cells at every $\alpha$, medium-replay is strongly $\alpha$-responsive ($7$ of $9$ Frozen ties at $\alpha=0.9$), and on AntMaze a BC frozen actor solves zero episodes so every composition rule inherits zero. CQL/IQL-guided and SF/GPI baselines (appendix-only) are competitive in places but do not consistently overturn this picture. PoE and KL-Reg therefore position as a single actor-anchored safety layer with a regime boundary set by actor competence rather than knob choice.
\end{abstract}

\section{Introduction}

Offline reinforcement learning (RL) \citep{janner2021offline} trains policies from previously collected data without environment interaction \citep{chenreddy2026epistemic, park2025pretraining, zhang2026reform}. Deployment objectives often change after training, but existing offline RL methods address such changes during training through conservative value estimation or behavior-regularized optimization \citep{chen2023conservative, li2023crop, guan2024poce}. A practically important setting remains less explored: the actor has already been trained and approved, the objective is revised afterward, and modifying the actor is not feasible. This frozen-actor setting arises when data access is limited, retraining is expensive, or any policy change requires governance or revalidation.

We consider \emph{actor-preserving deployment-time adaptation}: steering a frozen offline actor toward a revised objective while remaining conservative relative to the validated policy. We combine a fixed actor $\pi_\theta(a\mid s)$ with a goal-conditioned deployment prior $\rho_\phi(a\mid s,g)$ through a Product-of-Experts (PoE) refinement \citep{welling2007product}. The refined policy is the exact solution to an entropy-regularized objective, and for diagonal-Gaussian actors and priors it reduces to a closed-form precision-weighted update. As a clarifying unification in the frozen-actor deployment setting we make explicit that this closed form is algebraically the same deterministic policy as the KL-regularized frozen-actor adaptation rule under $\beta=\alpha/(1-\alpha)$; the underlying Gaussian KL-barycenter algebra is standard, but the two rules have been treated as distinct mechanisms in prior work, and we organize the empirical analysis around this identity.

Empirically, we study this actor-anchored composition on four D4RL continuous-control environments using MuJoCo rollouts ($3{,}900$ episodes; $95\%$ bootstrap CIs) \citep{fu2020d4rl,todorov2012mujoco} and then extend to six harder cells and two AntMaze diagnostics. The central practical finding is \emph{graceful degradation}: under degraded or random priors, precision-weighted composition remains close to the frozen actor while prior-only and additive adaptation collapse, and a KL-budget selector typically lands near the oracle operating point. An $\alpha\in\{0.1,0.3,0.5,0.7,0.9\}$ sweep then exposes an \emph{actor-competence ceiling} (medium-expert \textsc{Hurt} in all $9$ cells at every $\alpha$; AntMaze with a BC actor zero across all methods), so composition is bounded by actor competence rather than knob choice.

\paragraph{Contributions.}
\begin{itemize}[leftmargin=1.2em]
    \item \textbf{Closed-form unification.} We formalize \emph{actor-preserving deployment-time adaptation} and make explicit that the closed-form Gaussian PoE refinement $\pi_{\mathrm{ref}}\propto \pi_\theta^{\alpha}\rho_\phi^{1-\alpha}$ is the same deterministic policy as the KL-regularized adaptation update with $\beta=\alpha/(1-\alpha)$, with the exact covariance relationship $\Sigma_{\mathrm{ref}}^{\mathrm{PoE}}(\alpha) = (1+\beta)\,\Sigma_{\mathrm{ref}}^{\mathrm{KL\text{-}Reg}}(\beta)$. The underlying Gaussian KL-barycenter algebra is standard; the contribution is to make this identity explicit in the frozen-actor deployment setting, where the two rules have been treated as separate mechanisms.

    \item \textbf{Practical robustness under degraded priors.} The central empirical finding is graceful degradation, not universal gain: PoE/KL-Reg anchor to the frozen actor while additive and prior-only collapse as the prior is made noisy or random, and a KL-budget rule recovers near-oracle $\alpha$ without hindsight tuning. This holds on the same four D4RL envs used for the headline package ($3{,}900$ MuJoCo episodes, $95\%$ bootstrap CIs, seed-matched pairs).

    \item \textbf{Boundary characterization.} As an empirical characterization of when actor-anchored composition is worth applying, we report a \textsc{Help}/\textsc{Frozen}/\textsc{Hurt} split of $4$/$5$/$3$ on the four envs and extend to harder regimes. A full $\alpha$ sweep on six cells (medium-replay, medium-expert) plus AntMaze diagnostics exposes an \emph{actor-competence ceiling}: medium-expert is \textsc{Hurt} in all $9$ cells at every $\alpha$ we test, medium-replay recovers to Frozen ties at conservative $\alpha$, and AntMaze with a BC frozen actor yields zero success across all composition and critic-guided rules.

    \item \textbf{Supporting baselines and honest theory.} CQL/IQL-guided and SF/GPI baselines are benchmarked in the appendix and are competitive in places without consistently overturning the actor-anchored robustness picture. We include a conservative-improvement decomposition (Theorem~\ref{thm:cpi}) as a structural reading only; its penalty coefficient is loose at $\gamma{=}0.99$ and we do not claim predictive power.
\end{itemize}

\section{Related Work}
\label{sec:related-work}

\paragraph{Offline reinforcement learning.}
Offline RL controls distributional mismatch between the behavior and learned policies \citep{kumar2020conservative,fujimoto2019off,kostrikov2021offline}, and advantage-weighted regression / AWAC \citep{peng2019advantage,nair2020awac} give the closed-form KL-regularized update $\pi(a\mid s)\propto \pi_\beta(a\mid s)\,\exp(A(s,a)/\beta)$ as the canonical Gaussian trust-region reference. All of these adapt at training time; our setting holds the actor fixed and adapts only at deployment.

\paragraph{Transfer with changing reward.}
Successor features and generalized policy improvement handle transfer across reward weightings \citep{barreto2017successor,barreto2018transfer}, and goal-conditioned RL conditions policies or value functions on a goal at training time \citep{chane2021goal,liu2022goal}. These modify the training pipeline; we adapt only at deployment through a separate goal-conditioned prior.

\paragraph{Multiplicative composition of policies.}
Product-of-Experts combines distributions multiplicatively \citep{hinton2002training}. In RL, energy-based max-entropy policies \citep{haarnoja2017reinforcement} and composable deep RL \citep{haarnoja2018composable} use multiplicative composition of Gaussian policies with the closed-form precision-weighted mean and covariance we reuse; decoding-time composition of base and expert models has been studied for language generation \citep{liu2021dexperts}. We do not introduce PoE composition; the contribution is to make its identity with the KL-regularized frozen-actor update explicit and to use it to organize an empirical characterization.

\section{Problem Setup}
\label{sec:prelim}

Let $\mathcal{M}=(\mathcal{S},\mathcal{A},P,r,\gamma)$ be a discounted MDP and let $g \in \mathcal{G}$ denote a deployment goal that modifies the effective reward through $r_g(s,a)$. In our D4RL locomotion experiments, $g\in\mathbb{R}^3$ is a weight vector over the reward components (forward, control, alive) and $r_g(s,a)=g\cdot\mathit{rc}(s,a)$. The deployment objective is $J_g(\pi)=\mathbb{E}_{\tau\sim\pi}[\sum_{t\ge 0}\gamma^t r_g(s_t,a_t)]$.

We are given a fixed offline dataset
\[
\mathcal{D}=\{(s_i,a_i,r_i,s_i',d_i)\}_{i=1}^{N},
\]
collected by some behavior policy. Using only $\mathcal{D}$, we pretrain an actor $\pi_\theta(a\mid s)$ and then freeze its parameters. At deployment time, no new interaction, no new data collection, and no actor retraining are allowed. The problem is therefore not to learn a new policy from scratch, but to construct a deployment policy that remains conservative with respect to $\pi_\theta$ while adapting to the updated deployment goal $g$.

For a policy $\pi$ we write $V_g^\pi$, $Q_g^\pi$, and $A_g^\pi$ for the goal-conditioned value, action-value, and advantage functions, and $d^\pi(s)=(1-\gamma)\sum_{t\ge 0}\gamma^t \Pr(s_t=s\mid\pi)$ for the discounted state occupancy. We study goal shift as the main deployment mismatch in the experiments; train-to-deploy transition shift is addressed structurally in Appendix~\ref{app:theory-extended}.

\section{Deployment-Time Policy Refinement via Product-of-Experts}
\label{sec:method}

We adapt a frozen offline actor at deployment time by composing it with a goal-conditioned prior. The central idea is to modify action selection without modifying actor parameters.

\subsection{PoE Refinement}
\label{sec:poe}

Let $\pi_\theta(a\mid s)$ denote the frozen actor and let $\rho_\phi(a\mid s,g)$ denote a goal-conditioned deployment prior encoding the updated preference under goal $g$. We define the refined deployment policy as
\begin{equation}
\pi_{\mathrm{ref}}(a\mid s,g)
=
\frac{
\pi_\theta(a\mid s)^{\alpha}\rho_\phi(a\mid s,g)^{1-\alpha}
}{
\sum_{a'\in\mathcal{A}}
\pi_\theta(a'\mid s)^{\alpha}\rho_\phi(a'\mid s,g)^{1-\alpha}
},
\label{eq:poe-policy}
\end{equation}
where $\alpha\in[0,1]$ controls the adaptation--conservatism tradeoff. As $\alpha\to1$, the refined policy recovers the frozen actor; as $\alpha\to0$, it approaches the deployment prior restricted to actor-supported actions.

This multiplicative form is actor-preserving in finite action spaces: for any $\alpha>0$, if $\pi_\theta(a\mid s)=0$, then $\pi_{\mathrm{ref}}(a\mid s,g)=0$. Thus, refinement cannot assign mass to actions ruled out by the frozen actor.

\subsection{Goal-Conditioned Prior}
\label{sec:prior}

Our primary instantiation uses a lightweight goal-conditioned network that maps $(s,g)$ to a distribution over actions or, in continuous control, to the parameters of a Gaussian density. The prior is trained offline and separately from the actor, so deployment-time adaptation requires only recombining the frozen actor and the learned prior under a specified goal $g$.

This design keeps actor pretraining and deployment-time adaptation modular: the actor captures behavior supported by the offline dataset, while the prior provides goal-specific steering at inference time.

\subsection{Continuous-Action Gaussian Form}
\label{sec:continuous-actions}

In the D4RL continuous-control setting, both the frozen actor and the deployment prior are diagonal Gaussian policies,
\[
\pi_\theta(\cdot\mid s)=\mathcal{N}(\mu_\theta(s),\Sigma_\theta(s)),
\qquad
\rho_\phi(\cdot\mid s,g)=\mathcal{N}(\mu_\phi(s,g),\Sigma_\phi(s,g)).
\]
The PoE-refined policy is then Gaussian with precision and mean
\[
\Sigma_{\mathrm{ref}}^{-1}
=
\alpha\,\Sigma_\theta^{-1}
+
(1-\alpha)\,\Sigma_\phi^{-1},
\qquad
\mu_{\mathrm{ref}}
=
\Sigma_{\mathrm{ref}}
\bigl(
\alpha\,\Sigma_\theta^{-1}\mu_\theta
+
(1-\alpha)\,\Sigma_\phi^{-1}\mu_\phi
\bigr),
\]
a closed-form precision-weighted combination of the frozen actor and the deployment prior.

\paragraph{Equivalence with KL-regularized adaptation.}
\label{sec:equivalence}
The Gaussian PoE update above coincides with the closed-form KL-regularized frozen-actor adaptation update $\arg\min_\pi \mathbb{E}_\pi[-\log\rho_\phi(a\mid s,g)] + \beta\,\mathrm{KL}(\pi\|\pi_\theta)$, which has Gaussian precision $\beta\Sigma_\theta^{-1}+\Sigma_\phi^{-1}$ and mean $(\beta\Sigma_\theta^{-1}\mu_\theta+\Sigma_\phi^{-1}\mu_\phi)/(\beta\Sigma_\theta^{-1}+\Sigma_\phi^{-1})$. Factoring $(1-\alpha)$ out of the PoE mean and setting $\beta=\alpha/(1-\alpha)$ recovers this expression exactly. The refined means agree, and the refined covariances satisfy $\Sigma_{\mathrm{ref}}^{\mathrm{PoE}}(\alpha) = (1+\beta)\,\Sigma_{\mathrm{ref}}^{\mathrm{KL\text{-}Reg}}(\beta)$: a global scalar that cancels under deterministic action selection. Under the deterministic deployment used in our rollouts, PoE$(\alpha)$ and KL-Reg$(\beta=\alpha/(1-\alpha))$ are the same policy. Appendix~\ref{app:equivalence_audit} reports a dataset-state and rollout audit.

\section{Theoretical Analysis}
\label{sec:theory}

We include two theorems for structural interpretation only; empirical claims in the main body are based on rollouts, not on the right-hand sides of these bounds. Theorem~\ref{thm:variational} is a standard KL-barycenter identity that ties the coefficient $\alpha$ to a weighted KL objective. Theorem~\ref{thm:cpi} decomposes the refined policy's return into a one-step advantage gain and a total-variation deviation penalty; the resulting bound is loose at standard discount factors (Appendix~\ref{app:cpi_diagnostic}), so we use it as a qualitative gain--deviation reading of the composition, not as a predictor of return. Both are stated in the finite-action setting for clarity, with a Wasserstein continuous-action analogue and additional guarantees in Appendix~\ref{app:theory-extended}.

\subsection{Variational Characterization}
\label{sec:theory-variational}

For fixed state $s$ and deployment goal $g$, PoE refinement is the unique optimizer of an entropy-regularized objective that trades off attraction to the frozen actor and attraction to the deployment prior. The identity is the standard KL-barycenter statement; we include it to ground the Gaussian closed form of Sec.~\ref{sec:continuous-actions} and to tie $\alpha$ to a weighted KL objective.

\begin{theorem}[Variational optimality]
\label{thm:variational}
Fix $s$ and $g$. Assume a finite action space and let $\Gamma_{s,g}=\{a\in\mathcal{A}:\pi_\theta(a\mid s)>0,\ \rho_\phi(a\mid s,g)>0\}$ be the effective common support of the actor and prior, on which both log-densities are well defined. Optimization is restricted to distributions supported on $\Gamma_{s,g}$; actions outside $\Gamma_{s,g}$ receive zero mass in any maximizer because the weighted KL objective is otherwise infinite and the PoE numerator vanishes (see the proof in Appendix~\ref{app:proof-theorem1}). Then the unique maximizer of
\begin{equation}
\mathcal{F}_{s,g}(\pi)
=
\mathbb{E}_{a\sim\pi}\!\left[
\alpha\log\pi_\theta(a\mid s)
+(1-\alpha)\log\rho_\phi(a\mid s,g)
\right]
+\mathcal{H}(\pi)
\label{eq:variational-obj}
\end{equation}
is the PoE policy in Eq.~\eqref{eq:poe-policy}. Equivalently,
\[
\pi_{\mathrm{ref}}
=
\operatornamewithlimits{argmin}_{\pi}
\Big[
\alpha\,\mathrm{KL}(\pi\|\pi_\theta)
+
(1-\alpha)\,\mathrm{KL}(\pi\|\rho_\phi)
\Big].
\]
\end{theorem}

Theorem~\ref{thm:variational} makes the multiplicative composition rule a variational consequence rather than a heuristic; the identity itself is textbook KL-barycenter. The continuous-action Gaussian closed form of Sec.~\ref{sec:continuous-actions} is its precision-space solution, and the equivalence with KL-regularized adaptation (Sec.~\ref{sec:equivalence}) follows directly. In finite action spaces, the same identity preserves actor support: actions ruled out by the frozen actor receive zero mass in $\pi_{\mathrm{ref}}$ for every $\alpha>0$.

\subsection{Conservative Improvement Relative to the Frozen Actor}
\label{sec:theory-cpi}

We next characterize how refinement affects return under the deployment objective. Let
\[
\bar{A}_g^{\pi_\theta}(s)
=
\mathbb{E}_{a\sim\pi_{\mathrm{ref}}(\cdot\mid s,g)}
\left[A_g^{\pi_\theta}(s,a)\right]
\]
denote the one-step expected advantage of the refined policy under the frozen actor's value function, and define
\[
\delta_\pi
=
\sup_s
\mathrm{TV}\!\Big(
\pi_{\mathrm{ref}}(\cdot\mid s,g),
\pi_\theta(\cdot\mid s)
\Big).
\]

\begin{theorem}[Conservative improvement decomposition]
\label{thm:cpi}
Assume bounded rewards and bounded prior log-density. Then
\begin{equation}
J_g(\pi_{\mathrm{ref}})-J_g(\pi_\theta)
\;\ge\;
\frac{1}{1-\gamma}
\mathbb{E}_{s\sim d^{\pi_\theta}}
\!\left[\bar{A}_g^{\pi_\theta}(s)\right]
-
\frac{2\gamma}{(1-\gamma)^2}\,\varepsilon_A\,\delta_\pi,
\label{eq:cpi-bound}
\end{equation}
where
\[
\varepsilon_A
=
\sup_s \left|\bar{A}_g^{\pi_\theta}(s)\right|.
\]
\end{theorem}

The decomposition makes the one-step gain versus deviation tradeoff explicit; as noted above, at $\gamma=0.99$ the penalty coefficient $2\gamma/(1-\gamma)^2=19{,}800$ with empirical sup-state TV near $1$, so we use Theorem~\ref{thm:cpi} only as a qualitative reading. Proofs are in Appendix~\ref{app:theory-extended}.

\section{Experimental Setup}
\label{sec:setup}

We evaluate on four D4RL continuous-control environments (\texttt{halfcheetah-medium-v2}, \texttt{hopper-medium-v2}, \texttt{walker2d-medium-v2}, \texttt{hopper-medium-expert-v2}). A diagonal-Gaussian actor is trained by behavioral cloning and then frozen; a goal-conditioned Gaussian prior is trained separately by contrastive goal-weighted behavioral cloning. Their PoE composition is Gaussian in closed form (Sec.~\ref{sec:continuous-actions}). For the harder-regime analysis (Sec.~\ref{sec:harder_regimes}, Appendix~\ref{app:harder_rollouts}) we additionally train and evaluate on the matching \texttt{-medium-replay-v2} and \texttt{-medium-expert-v2} variants of halfcheetah, hopper, and walker2d, plus \texttt{antmaze-umaze-v2} and \texttt{antmaze-medium-play-v2} as diagnostic boundary cases. All rollouts use the same protocol as the headline package unless noted.

\paragraph{Evaluation protocol.}
All main-body returns come from real MuJoCo simulator rollouts: five seeds $\times$ five episodes per seed, $25$ episodes per (method, environment, goal) cell, $3{,}900$ episodes total. Deployment goals are three-dimensional weight vectors over the per-step reward components (forward reward, control cost, alive bonus): $G_1\!=\!(1,0.1,0.1)$ speed, $G_2\!=\!(0.5,0.5,0.5)$ balanced, $G_3\!=\!(0.1,1,0.1)$ efficient. Reward components are recovered from the MuJoCo reward using the same decomposition used at prior training, so the goal-weighted return $\widehat{J}_g(\pi)=\sum_{t} g\cdot rc_t$ and the prior's objective are aligned. We also report D4RL-normalized return, and $\mathrm{KL}(\pi(\cdot\mid s,g)\,\|\,\pi_\theta(\cdot\mid s))$ averaged along rollout states.

\paragraph{Methods.}
We run Frozen Actor, Prior Only, Additive Mix ($\lambda=0.5$ in mean-and-std), matched KL-Reg at $\beta\in\{0.111,0.429,1.000,2.333,9.000\}$, and PoE at $\alpha\in\{0.1,0.3,0.5,0.7,0.9\}$ (with $\beta=\alpha/(1-\alpha)$). An AWR/AWAC-style \citep{peng2019advantage,nair2020awac} first-order critic-gradient step $\mu_\theta(s)+(\sigma_\theta^2(s)/\beta)\,\nabla_a Q_g(s,\mu_\theta(s))$ at $\beta\in\{0.5,1,3\}$ and representative CQL/IQL-guided steps are appendix-only comparisons (Appendices~\ref{app:awr_baseline}, \ref{app:cql_iql_baselines}).

\paragraph{Statistical reporting.}
Per-cell $95\%$ bootstrap CIs over $25$ episodes, seed-matched paired differences with $95\%$ bootstrap CIs, and rliable-style \citep{agarwal2021deep} probability of improvement, all resampled from per-seed summaries. ``Return'' means the MuJoCo rollout return throughout the main body.

\section{Results}
\label{sec:results}

From real MuJoCo rollouts we report three findings, visualized in Figures~\ref{fig:equivalence}, \ref{fig:headline_rollout}, and \ref{fig:prior_degradation} in that order: (i)~PoE refinement and KL-regularized adaptation are the same deterministic Gaussian policy under $\alpha\leftrightarrow\beta=\alpha/(1-\alpha)$; (ii)~composition with a reliable prior helps on some cells, ties the frozen actor on others, and hurts on the remainder; (iii)~under degraded priors, precision-weighted composition anchors to the frozen actor while additive and prior-only adaptation collapse. Those prior-degradation results are the strongest practical evidence in the paper, because they show when the conservative anchor matters most.

\subsection{PoE Refinement and KL-Regularized Adaptation Are the Same Update}
\label{sec:d4rl_equivalence}

We verify the identity of Sec.~\ref{sec:equivalence} at two levels. On $5{,}000$ sampled dataset states per environment and every $\alpha\in\{0.1,0.3,0.5,0.7,0.9\}$, per-state $|\mu_{\mathrm{PoE}}-\mu_{\mathrm{KL\text{-}Reg}}|$ is at float-$32$ precision ($\le 2\times 10^{-6}$). In MuJoCo rollouts at $\alpha=0.5,\beta=1.0$ the matched pair produces bit-identical returns and episode lengths in all $12$ cells; at the other matched $\alpha$, paired $95\%$ CIs include zero and residuals are consistent with chaotic amplification of float-precision jitter, not a policy-level difference. Figure~\ref{fig:equivalence} summarizes both views; the full audit is in Appendix~\ref{app:equivalence_audit}.

\begin{figure*}[t]
\centering
\includegraphics[width=0.95\textwidth]{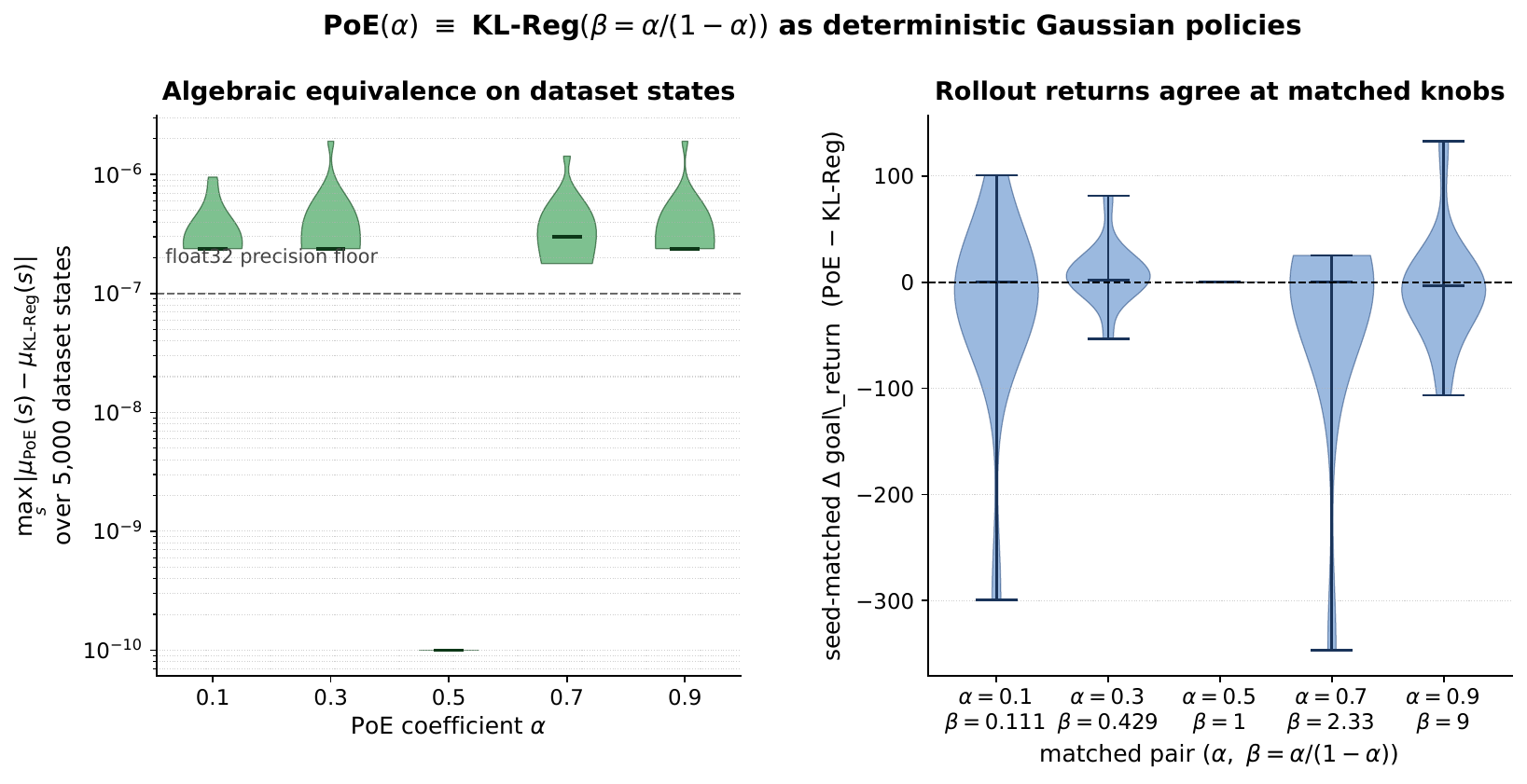}
\caption{Empirical verification that $\mathrm{PoE}(\alpha)$ and $\mathrm{KL\text{-}Reg}(\beta=\alpha/(1-\alpha))$ are the same deterministic Gaussian policy. \textbf{Left}: per-state $|\mu_{\mathrm{PoE}}-\mu_{\mathrm{KL\text{-}Reg}}|$ over $5{,}000$ dataset states per $\alpha$, across four D4RL environments and three deployment goals. Violin scale is logarithmic; the dashed line marks float-$32$ machine precision. \textbf{Right}: seed-matched rollout goal-weighted-return difference (PoE minus matched KL-Reg) at each matched $\alpha$, pooled across all $12$ environment-by-goal cells. The residual spread away from $\alpha=0.5$ is MuJoCo chaotic amplification of float-precision action jitter, not a policy-level difference.}
\label{fig:equivalence}
\end{figure*}

\subsection{Headline Rollout Results}
\label{sec:d4rl_headline}

Figure~\ref{fig:headline_rollout} reports the mean goal-weighted return with $95\%$ bootstrap CI per method family and per $(\text{env},\text{goal})$ cell. For each composition family (Additive, KL-Reg, PoE) we plot the best-in-family operating point on goal-weighted return; exact per-cell values are in Appendix~\ref{app:d4rl_single_env_detail}.

\begin{figure*}[t]
\centering
\includegraphics[width=0.95\textwidth]{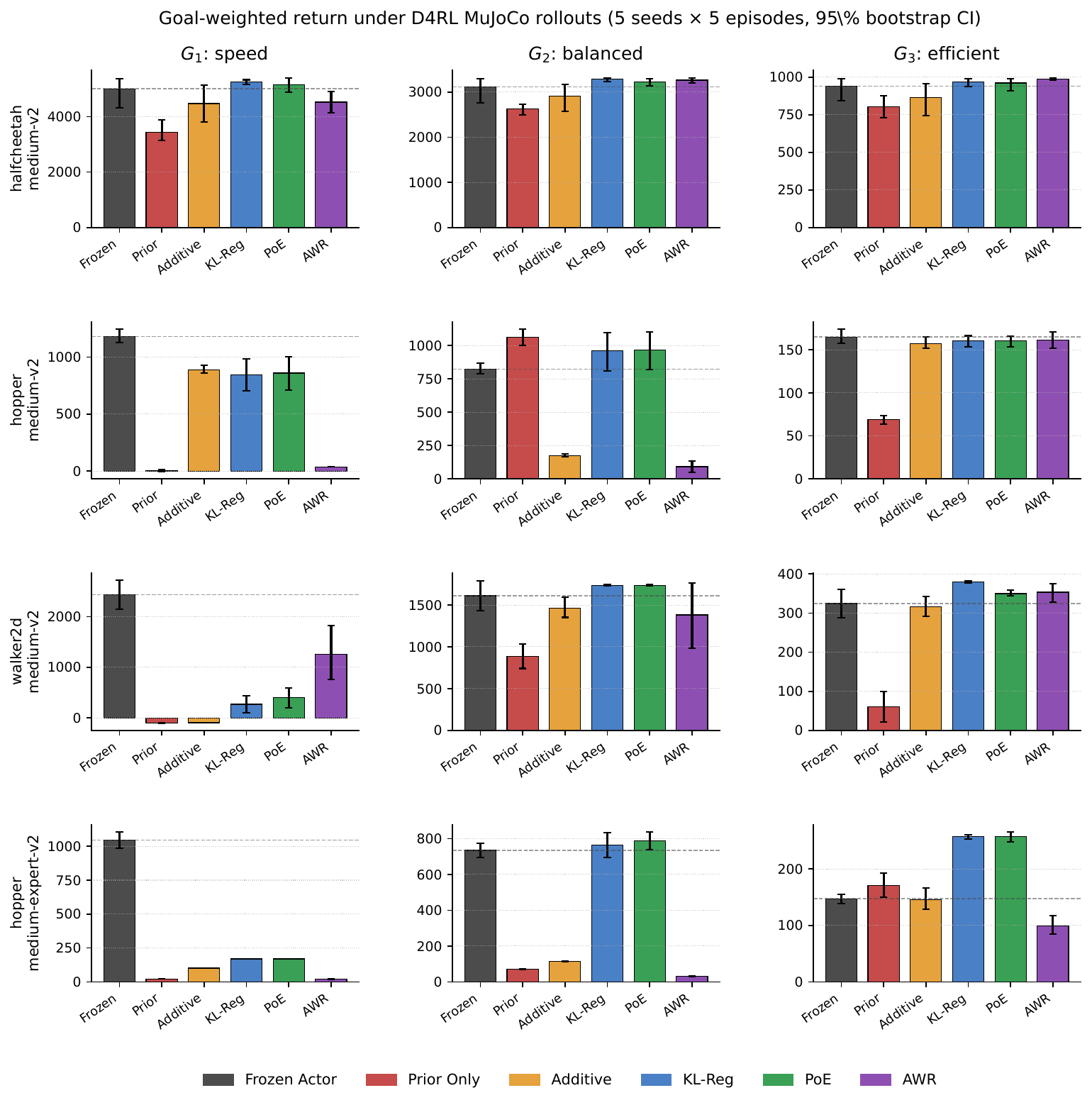}
\caption{Headline D4RL MuJoCo rollout comparison. Bars are means and error bars are $95\%$ bootstrap confidence intervals over $5$ seeds $\times\, 5$ episodes $=\,25$ episodes per cell. Rows are environments; columns are deployment goals. For each composition family, the plotted operating point is the best-in-family point on goal-weighted return.}
\label{fig:headline_rollout}
\end{figure*}

A seed-matched CI-half-width classifier (Appendix~\ref{app:diagnostic}) labels each cell \textsc{Help}/\textsc{Frozen}/\textsc{Hurt}; verdicts split $4$/$5$/$3$. On \texttt{halfcheetah-medium-v2} composition sits within overlapping CIs of the frozen actor on all three goals (\textsc{Frozen}). On \texttt{hopper-medium-v2} $G_2$, \texttt{hopper-medium-expert-v2} $G_2$, $G_3$, and \texttt{walker2d-medium-v2} $G_3$, the best composition method beats the frozen actor by $+144$, $+55$, $+110$, $+55$ (\textsc{Help}). On the three $G_1$ speed-dominant cells the frozen actor strictly dominates every composition method ($-291$, $-878$, $-2{,}030$): the actor is already near the velocity ceiling and redirection destabilizes the agent. All three \textsc{Hurt} cells are $G_1$ goals.

\subsection{Critic-Gradient Baseline (AWR/AWAC)}
\label{sec:d4rl_awr}

An AWR/AWAC-style baseline replaces the refinement prior with a first-order step along $\nabla_a Q_g(s,\mu_\theta(s))$ (Appendix~\ref{app:awr_baseline}). It beats Frozen on $3$ of $12$ cells (halfcheetah $G_2$/$G_3$, walker2d $G_3$), falls below Frozen by $>\!100$ points on $7$, and collapses to near Prior Only on $5$ hopper- and walker2d-$G_1$ cells: offline FQE is reliable in-distribution, but the AWR shift pushes $\mu_\theta(s)$ off the support ridge and early-terminates hopper-type agents. Density-anchored composition avoids this pathology. Representative CQL/IQL-guided steps at $\beta=1.0$ are more stable (CQL and IQL beat Frozen on $4$ and $5$ of $12$ cells) but neither consistently dominates the actor-anchored baseline (Appendix~\ref{app:cql_iql_baselines}).

\subsection{Robustness to Degraded Deployment Priors}
\label{sec:d4rl_prior_robustness_main}

We evaluate the four environments and three goals under four prior variants (the trained checkpoint, \emph{undertrained}, \emph{noisy} with Gaussian parameter noise at $\sigma=0.05$, and \emph{random}) with three seeds and three episodes per cell. Figure~\ref{fig:prior_degradation} reports the gap in goal-weighted return relative to the frozen actor, averaged across goals.

\begin{figure*}[t]
\centering
\includegraphics[width=0.95\textwidth]{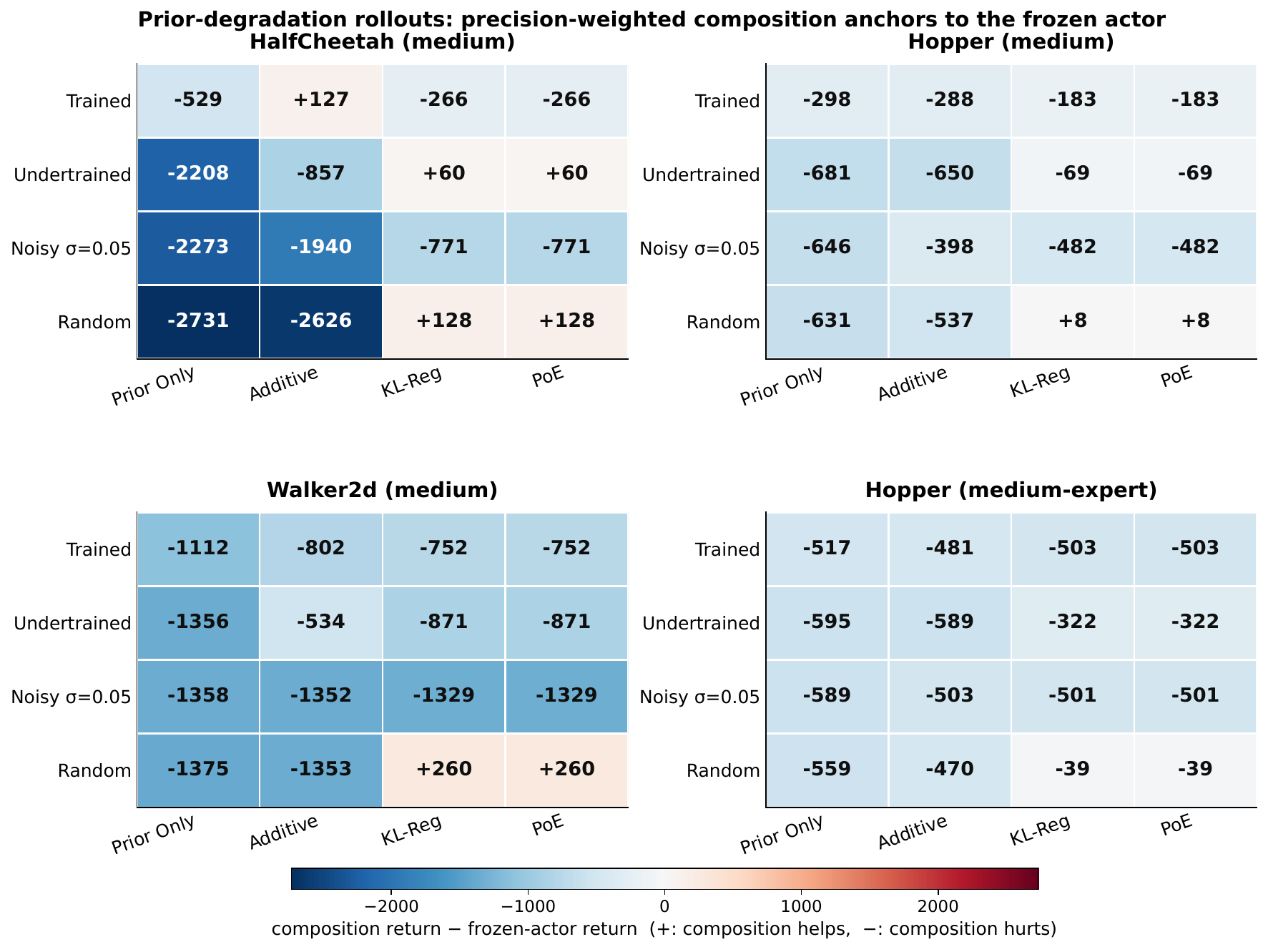}
\caption{Prior-degradation rollouts. Each cell is the mean goal-weighted return of a composition family minus the frozen-actor return on the same environment, averaged across goals. Blue: composition helps. Red: composition hurts. PoE and matched KL-Reg are the same Gaussian update (Sec.~\ref{sec:d4rl_equivalence}); a single representative column is shown.}
\label{fig:prior_degradation}
\end{figure*}

Additive and prior-only deployment collapse as the prior degrades, while precision-weighted composition stays near the frozen actor. Under \emph{random} prior: Prior Only $86.3$, Additive $164.2$, Frozen $1{,}410.5$, PoE/KL-Reg $1{,}499.8$; under \emph{noisy} ($\sigma=0.05$): $193.8/362.0/\text{--}/639.6$; under \emph{undertrained}: $200.4/753.3/\text{--}/1{,}110.0$ (per-cell numbers: Appendix~\ref{app:d4rl_prior_quality}). A degraded prior has low precision, so the posterior mean shrinks toward the frozen actor; additive mixing carries damaged means through. Precision-weighted composition is an actor-anchored \emph{safety layer} under untrustworthy priors, not a return-maximization rule. A practical KL-budget $\alpha$ selector (Appendix~\ref{app:d4rl_adaptive_alpha}) typically lands at or near the oracle operating point, making $\alpha$ a deployable knob rather than an oracle-only curiosity.

\subsection{Harder Regimes and the Actor-Competence Ceiling}
\label{sec:harder_regimes}

To test whether the \textsc{Help}/\textsc{Frozen}/\textsc{Hurt} boundary depends on regime or $\alpha$, we swept $\alpha\in\{0.1,0.3,0.5,0.7,0.9\}$ on six harder cells (\texttt{halfcheetah/hopper/walker2d-medium-replay-v2} and \texttt{-medium-expert-v2}, $5{,}850$ episodes). Equivalence persists: per-cell $|\mu_{\mathrm{PoE}}-\mu_{\mathrm{KL\text{-}Reg}}|$ is exactly $0$ at $\alpha=0.5$ in all $18$ cells. Table~\ref{tab:harder_regime_verdict_alpha} gives verdict counts. Medium-replay is $\alpha$-responsive: $7$ of $9$ cells are Frozen ties at $\alpha=0.9$. Medium-expert is not: all $9$ cells remain \textsc{Hurt} at every $\alpha$, including $\alpha=0.9$ (refined policy $90\%$ frozen). On \texttt{walker2d-medium-expert} $G_1$ Frozen reaches $4920$ while PoE at $\alpha\in\{0.1,0.3,0.5,0.7,0.9\}$ collapses to $\{399,72,62,63,78\}$ with early termination. We read this as an \emph{actor-competence ceiling}: once the frozen actor is at the locomotion ceiling, the knob has no room to leave the support ridge. Per-cell values are in Appendix~\ref{app:harder_alpha_sweep}.

\begin{table}[tb]
\centering
\small
\caption{Harder-regime $\alpha$ sweep: \textsc{Help}/\textsc{Frozen}/\textsc{Hurt} verdict counts per regime and $\alpha$, using the CI-half-width classifier of Appendix~\ref{app:diagnostic}. Each row uses $9$ env-by-goal cells (three envs, three goals). Medium-expert stays uniformly \textsc{Hurt}; medium-replay approaches Frozen ties as $\alpha$ rises.}
\label{tab:harder_regime_verdict_alpha}
\vspace{2pt}
\begin{tabular}{lccccc|ccccc}
\toprule
\rowcolor[RGB]{232,236,241}
& \multicolumn{5}{c|}{\textbf{medium-replay}} & \multicolumn{5}{c}{\textbf{medium-expert}} \\
\rowcolor[RGB]{232,236,241}
$\alpha$ & $0.1$ & $0.3$ & $0.5$ & $0.7$ & $0.9$ & $0.1$ & $0.3$ & $0.5$ & $0.7$ & $0.9$ \\
\midrule
\textsc{Help}   & 0 & 0 & 0 & 1 & 0 & 1 & 0 & 1 & 1 & 0 \\
\textsc{Frozen} & 1 & 2 & 5 & 5 & \textbf{7} & 0 & 1 & 0 & 0 & \textbf{0} \\
\textsc{Hurt}   & 8 & 7 & 4 & 3 & 2 & 8 & 8 & 8 & 8 & \textbf{9} \\
\bottomrule
\end{tabular}
\end{table}

\paragraph{AntMaze as a diagnostic boundary case.}
We include \texttt{antmaze-umaze-v2} and \texttt{antmaze-medium-play-v2} as a frozen-actor-competence diagnostic rather than as a task-success benchmark: a BC frozen actor solves $0$ of $125$ episodes per env, and every deployment method (composition, critic-guided, AWR) inherits this zero (Appendix~\ref{app:antmaze_diagnostic}). This establishes the extreme of the actor-competence ceiling and is not a claim about PoE performance on AntMaze-style tasks. An SF/GPI baseline \citep{barreto2017successor,barreto2018transfer} is competitive on halfcheetah but fragile on hopper and walker2d; we keep it appendix-only (Appendix~\ref{app:sf_gpi}).

\section{Discussion}
\label{sec:discussion}

Three readings summarize the evidence. First, PoE and KL-regularized adaptation are the same deterministic Gaussian policy under $\beta=\alpha/(1-\alpha)$, so the choice between them is a reparameterization. Second, the method's central practical value is \emph{graceful degradation under degraded priors}: PoE/KL-Reg anchor to the frozen actor while additive and prior-only adaptation collapse, giving an actor-anchored safety layer rather than a return-maximization rule. Third, the harder-regime $\alpha$ sweep identifies an \emph{actor-competence ceiling}: on medium-expert every $\alpha$ in our grid is \textsc{Hurt} in every cell, so composition is bounded by actor competence rather than knob choice, and AntMaze with a BC actor is the extreme of the same phenomenon. Stronger critic-guided baselines (AWR, CQL, IQL) and an SF/GPI comparison sharpen the regime boundary rather than overturn it: CQL wins a handful of harder-regime cells, IQL is either a numerical no-op or collapses on hopper-ME, and AWR remains fragile under hopper-type early termination.

\paragraph{Diagnosing unsafe steering \emph{a priori}.}
The actor-competence ceiling we document is an empirical finding; we do not yet provide a formal predictor of whether a given frozen actor is safely steerable. Candidate indicators that a practitioner could compute without on-policy rollouts, and that our evidence is qualitatively consistent with, include (i) baseline return near the environment's return ceiling (medium-expert saturates and composition has no slack), (ii) high sensitivity of the frozen-actor return to small steering perturbations (early termination in hopper- and walker2d-ME under $\alpha{=}0.7$ despite a $70\%$ actor weight), (iii) large log-density mismatch between $\pi_\theta(\cdot\mid s)$ and $\rho_\phi(\cdot\mid s,g)$ on dataset states, and (iv) a dataset collected by BC-like imitation of a policy whose offline success rate is near zero (the AntMaze boundary). Turning these into a calibrated pre-deployment safety check is an open direction rather than a contribution of this paper.

\section{Conclusion}
\label{sec:conclusion}

We make explicit that PoE and KL-regularized frozen-actor adaptation are the same deterministic Gaussian policy under $\beta=\alpha/(1-\alpha)$, a clarifying unification in this deployment setting rather than a new technical mechanism. The paper's central practical value is graceful degradation: under degraded or random priors, actor-anchored precision weighting stays close to the frozen actor while additive and prior-only adaptation collapse, and a KL-budget rule recovers a near-oracle operating point. As an empirical characterization of where composition is worth applying, four D4RL envs give a $4$/$5$/$3$ \textsc{Help}/\textsc{Frozen}/\textsc{Hurt} split; an $\alpha$ sweep on six harder cells plus AntMaze diagnostics identifies an \emph{actor-competence ceiling} (medium-expert is uniformly \textsc{Hurt} at every $\alpha$, AntMaze with a BC actor yields zero success). The paper therefore positions PoE and KL-Reg as a single actor-anchored safety layer with a regime boundary set by actor competence, not as a new composition rule or a universal return-maximization recipe.

\paragraph{Limitations.}
Formal guarantees are finite-action; experiments use continuous-action diagonal-Gaussians with a Wasserstein analogue in Appendix~\ref{app:theory-continuous}. Theorem~\ref{thm:cpi} is structural only: penalty coefficient $19{,}800$ at $\gamma=0.99$ with empirical sup-state TV near $1$. Composition does not universally improve over Frozen (Hurt in $3$ of $12$ main-body cells and $15$ of $18$ harder-regime cells at the best per-cell $\alpha$). We do not probe $\alpha>0.9$, but $9/9$ \textsc{Hurt} at $\alpha=0.9$ on medium-expert indicates a structural ceiling, not a tuning failure. AntMaze is diagnostic only (BC actor returns zero); we do not claim the paradigm covers AntMaze-grade exploration. AWR/CQL/IQL/SF/GPI baselines are mixed and do not consistently overturn the actor-anchored robustness picture. Our budget is $5$ seeds $\times\,5$ episodes per cell; larger budgets would tighten the reported intervals.

\bibliographystyle{plainnat}
\bibliography{ref}


\clearpage
\appendix
\renewcommand\thesection{\Alph{section}}
\renewcommand\thesubsection{\thesection.\arabic{subsection}}

\label{sec:final-appendix}
\addcontentsline{toc}{section}{Supplementary Material}
\tableofcontents

\newpage

\section{Theoretical Analysis (Extended)}
\label{app:theory-extended}

\subsection{Proof of Theorem~\ref{thm:variational}}
\label{app:proof-theorem1}

Fix a state $s$ and goal $g$, and abbreviate
\[
p_a := \pi(a\mid s),\quad
u_a := \pi_\theta(a\mid s),\quad
v_a := \rho_\phi(a\mid s,g).
\]
Because logarithms are only defined on positive mass, we work on the effective support
\[
\Gamma_{s,g}
:=
\{a\in\mathcal{A}:\; u_a>0,\; v_a>0\}.
\]
For actions outside $\Gamma_{s,g}$, the weighted KL objective is either infinite or the Product-of-Experts numerator is zero, so such actions receive zero mass in the optimizer. It is therefore sufficient to optimize over probability vectors supported on $\Gamma_{s,g}$.

The per-state objective is
\[
\mathcal{F}_{s,g}(p)
=
\sum_{a\in\Gamma_{s,g}}
p_a\big(\alpha \log u_a + (1-\alpha)\log v_a\big)
-
\sum_{a\in\Gamma_{s,g}} p_a \log p_a,
\]
subject to $p_a\ge 0$ and $\sum_{a\in\Gamma_{s,g}} p_a=1$.

Introduce a Lagrange multiplier $\lambda$ for the simplex constraint:
\[
\mathcal{L}(p,\lambda)
=
\sum_{a\in\Gamma_{s,g}}
p_a\big(\alpha \log u_a + (1-\alpha)\log v_a\big)
-
\sum_{a\in\Gamma_{s,g}} p_a \log p_a
+
\lambda\left(\sum_{a\in\Gamma_{s,g}} p_a - 1\right).
\]
Differentiating with respect to $p_a$ gives
\[
\frac{\partial \mathcal{L}}{\partial p_a}
=
\alpha \log u_a + (1-\alpha)\log v_a - (\log p_a + 1) + \lambda.
\]
Setting the derivative to zero yields
\[
\log p_a
=
\alpha \log u_a + (1-\alpha)\log v_a + \lambda - 1.
\]
Exponentiating,
\[
p_a
=
\exp(\lambda-1)\,u_a^\alpha v_a^{1-\alpha}.
\]
Because $\sum_{a\in\Gamma_{s,g}} p_a=1$, the normalizing constant is
\[
\exp(\lambda-1)
=
\left(\sum_{a'\in\Gamma_{s,g}} u_{a'}^\alpha v_{a'}^{1-\alpha}\right)^{-1}.
\]
Therefore,
\[
p_a
=
\frac{u_a^\alpha v_a^{1-\alpha}}
{\sum_{a'\in\Gamma_{s,g}} u_{a'}^\alpha v_{a'}^{1-\alpha}},
\]
which is exactly the Product-of-Experts policy on the effective support, with zero mass outside that support:
\[
\pi_{\mathrm{ref}}(a\mid s,g)
=
\frac{
\pi_\theta(a\mid s)^\alpha
\rho_\phi(a\mid s,g)^{1-\alpha}
}{
\sum_{a'}
\pi_\theta(a'\mid s)^\alpha
\rho_\phi(a'\mid s,g)^{1-\alpha}
}.
\]

To prove the equivalent KL form, consider
\[
\alpha \KL(\pi\|\pi_\theta) + (1-\alpha)\KL(\pi\|\rho_\phi).
\]
Expanding both KL terms over $\Gamma_{s,g}$ yields
\begin{align*}
&\alpha \sum_{a\in\Gamma_{s,g}} \pi(a\mid s)\log\frac{\pi(a\mid s)}{\pi_\theta(a\mid s)}
+
(1-\alpha)\sum_{a\in\Gamma_{s,g}} \pi(a\mid s)\log\frac{\pi(a\mid s)}{\rho_\phi(a\mid s,g)} \\
&=
\sum_{a\in\Gamma_{s,g}} \pi(a\mid s)\log \pi(a\mid s)
-
\sum_{a\in\Gamma_{s,g}} \pi(a\mid s)\big(\alpha\log\pi_\theta(a\mid s)+(1-\alpha)\log\rho_\phi(a\mid s,g)\big).
\end{align*}
This is exactly $-\mathcal{F}_{s,g}(\pi)$ up to the simplex constraint. Hence maximizing $\mathcal{F}_{s,g}$ is equivalent to minimizing the weighted KL objective. Uniqueness follows from strict concavity of entropy on the simplex.
\hfill $\square$

\begin{proposition}[Interpolation and support preservation]
\label{prop:interp}
For any fixed state $s$ and goal $g$, the PoE refinement interpolates between the frozen actor and the deployment prior as $\alpha\to 1$ and $\alpha\to 0$, and for every $\alpha>0$ it preserves actor support.
\end{proposition}

\subsection{Proof of Proposition~\ref{prop:interp}}
\label{app:proof-proposition1}

Recall that
\[
\pi_{\mathrm{ref}}(a\mid s,g)
\propto
\pi_\theta(a\mid s)^\alpha
\rho_\phi(a\mid s,g)^{1-\alpha}.
\]

As $\alpha\to 1$, the exponent on $\rho_\phi$ tends to zero, so for every action with $\pi_\theta(a\mid s)>0$,
\[
\pi_\theta(a\mid s)^\alpha \rho_\phi(a\mid s,g)^{1-\alpha}
\to
\pi_\theta(a\mid s).
\]
After normalization,
\[
\lim_{\alpha\to 1}\pi_{\mathrm{ref}}(\cdot\mid s,g)
=
\pi_\theta(\cdot\mid s).
\]

As $\alpha\to 0$, support preservation still holds for every $\alpha>0$. Thus the limit is the prior restricted to actor support:
\[
\lim_{\alpha\to 0}\pi_{\mathrm{ref}}(a\mid s,g)
=
\frac{
\rho_\phi(a\mid s,g)\,\mathbf{1}\{\pi_\theta(a\mid s)>0\}
}{
\sum_{a'} \rho_\phi(a'\mid s,g)\,\mathbf{1}\{\pi_\theta(a'\mid s)>0\}
}.
\]
In the special case where
\[
\supp(\rho_\phi(\cdot\mid s,g))
\subseteq
\supp(\pi_\theta(\cdot\mid s)),
\]
this reduces to
\[
\lim_{\alpha\to 0}\pi_{\mathrm{ref}}(\cdot\mid s,g)
=
\rho_\phi(\cdot\mid s,g).
\]

For support preservation, if $\pi_\theta(a\mid s)=0$, then
\[
\pi_\theta(a\mid s)^\alpha = 0
\qquad
\text{for every }\alpha>0,
\]
so the numerator in Eq.~\eqref{eq:poe-policy} is zero and therefore
\[
\pi_{\mathrm{ref}}(a\mid s,g)=0.
\]
\hfill $\square$

\subsection{Auxiliary Lemmas for Policy Improvement}
\label{app:auxiliary-lemmas}

We first state two standard ingredients.

\begin{lemma}[Performance difference lemma]
\label{lem:pdl}
For any two stationary policies $\pi$ and $\pi'$ and deployment goal $g$,
\[
J_g(\pi') - J_g(\pi)
=
\frac{1}{1-\gamma}
\mathbb{E}_{s\sim d^{\pi'}}
\mathbb{E}_{a\sim\pi'(\cdot\mid s)}
\big[A_g^\pi(s,a)\big].
\]
\end{lemma}

\begin{proof}
This is the standard performance difference lemma specialized to the goal-indexed reward $r_g$. It follows by telescoping the Bellman equations for $V_g^\pi$ along trajectories generated by $\pi'$.
\hfill $\square$
\end{proof}

\begin{lemma}[Occupancy difference bound]
\label{lem:occ}
Let
\[
\delta_\pi
=
\sup_{s}
\TV(\pi'(\cdot\mid s),\pi(\cdot\mid s)).
\]
Then
\[
\|d^{\pi'} - d^\pi\|_1
\le
\frac{2\gamma}{1-\gamma}\,\delta_\pi.
\]
\end{lemma}

\begin{proof}
This is a standard discounted occupancy perturbation bound obtained from the fixed-point equations for discounted visitation together with contraction and kernel perturbation arguments.
\hfill $\square$
\end{proof}

\subsection{Proof of Theorem~\ref{thm:cpi}}
\label{app:proof-theorem-2}

Starting from Lemma~\ref{lem:pdl},
\[
J_g(\pi_{\mathrm{ref}}) - J_g(\pi_\theta)
=
\frac{1}{1-\gamma}
\mathbb{E}_{s\sim d^{\pi_{\mathrm{ref}}}}
\big[\bar A_g^{\pi_\theta}(s)\big].
\]
Add and subtract the same quantity under $d^{\pi_\theta}$:
\begin{align*}
J_g(\pi_{\mathrm{ref}}) - J_g(\pi_\theta)
&=
\frac{1}{1-\gamma}
\mathbb{E}_{s\sim d^{\pi_\theta}}
\big[\bar A_g^{\pi_\theta}(s)\big] \\
&\quad+
\frac{1}{1-\gamma}
\left(
\mathbb{E}_{s\sim d^{\pi_{\mathrm{ref}}}}
\big[\bar A_g^{\pi_\theta}(s)\big]
-
\mathbb{E}_{s\sim d^{\pi_\theta}}
\big[\bar A_g^{\pi_\theta}(s)\big]
\right).
\end{align*}
By definition of $\varepsilon_A$,
\[
\left|
\bar A_g^{\pi_\theta}(s)
\right|
\le
\varepsilon_A
\quad
\text{for all } s.
\]
Therefore,
\[
\left|
\mathbb{E}_{s\sim d^{\pi_{\mathrm{ref}}}}
\big[\bar A_g^{\pi_\theta}(s)\big]
-
\mathbb{E}_{s\sim d^{\pi_\theta}}
\big[\bar A_g^{\pi_\theta}(s)\big]
\right|
\le
\varepsilon_A \|d^{\pi_{\mathrm{ref}}} - d^{\pi_\theta}\|_1.
\]
Applying Lemma~\ref{lem:occ},
\[
\|d^{\pi_{\mathrm{ref}}} - d^{\pi_\theta}\|_1
\le
\frac{2\gamma}{1-\gamma}\,\delta_\pi.
\]
Combining the inequalities gives
\[
J_g(\pi_{\mathrm{ref}}) - J_g(\pi_\theta)
\ge
\frac{1}{1-\gamma}
\mathbb{E}_{s\sim d^{\pi_\theta}}
\big[\bar A_g^{\pi_\theta}(s)\big]
-
\frac{2\gamma}{(1-\gamma)^2}\,
\varepsilon_A\,\delta_\pi.
\]
\hfill $\square$

\begin{corollary}[Sufficient condition for improvement]
\label{cor:improve}
If the expected local advantage term in Theorem~\ref{thm:cpi} dominates the conservatism penalty, then PoE improves over the frozen actor.
\end{corollary}

\subsection{Proof of Corollary~\ref{cor:improve}}
\label{app:proof-corollary-1}

The corollary follows immediately from Theorem~\ref{thm:cpi}. If
\[
\mathbb{E}_{s\sim d^{\pi_\theta}}
\big[\bar A_g^{\pi_\theta}(s)\big]
>
\frac{2\gamma}{1-\gamma}\,\varepsilon_A\,\delta_\pi,
\]
then the lower bound on
\[
J_g(\pi_{\mathrm{ref}})-J_g(\pi_\theta)
\]
is strictly positive.
\hfill $\square$

\subsection{Auxiliary Lemma for Transition Shift}
\label{app:auxiliary-transition-shift}

\begin{lemma}[State occupancy shift under kernel perturbation]
\label{lem:kernelshift}
Assume
\[
\sup_{s,a}
\TV(P_{\mathrm{train}}(\cdot\mid s,a),P_{\mathrm{deploy}}(\cdot\mid s,a))
\le
\varepsilon_P.
\]
Then for any stationary policy $\pi$,
\[
\|d_{\mathrm{deploy}}^\pi - d_{\mathrm{train}}^\pi\|_1
\le
\frac{2\gamma}{1-\gamma}\,\varepsilon_P.
\]
\end{lemma}

\begin{proof}
Fix a policy $\pi$ and define the policy-induced state kernels
\[
P_{\mathrm{train}}^\pi(s'\mid s)
=
\sum_a \pi(a\mid s)P_{\mathrm{train}}(s'\mid s,a),
\quad
P_{\mathrm{deploy}}^\pi(s'\mid s)
=
\sum_a \pi(a\mid s)P_{\mathrm{deploy}}(s'\mid s,a).
\]
By convexity of total variation,
\[
\sup_s \TV(P_{\mathrm{train}}^\pi(\cdot\mid s),P_{\mathrm{deploy}}^\pi(\cdot\mid s))
\le
\varepsilon_P.
\]
Now write the discounted occupancy equations:
\[
d_{\mathrm{train}}^\pi
=
(1-\gamma)\mu + \gamma (P_{\mathrm{train}}^\pi)^\top d_{\mathrm{train}}^\pi,
\qquad
d_{\mathrm{deploy}}^\pi
=
(1-\gamma)\mu + \gamma (P_{\mathrm{deploy}}^\pi)^\top d_{\mathrm{deploy}}^\pi,
\]
where $\mu$ is the initial state distribution. Subtracting,
\[
d_{\mathrm{deploy}}^\pi - d_{\mathrm{train}}^\pi
=
\gamma (P_{\mathrm{deploy}}^\pi)^\top (d_{\mathrm{deploy}}^\pi - d_{\mathrm{train}}^\pi)
+
\gamma\big((P_{\mathrm{deploy}}^\pi)^\top-(P_{\mathrm{train}}^\pi)^\top\big)d_{\mathrm{train}}^\pi.
\]
Taking $\ell_1$ norms and using that stochastic kernels are nonexpansive in $\ell_1$,
\[
\|d_{\mathrm{deploy}}^\pi - d_{\mathrm{train}}^\pi\|_1
\le
\gamma \|d_{\mathrm{deploy}}^\pi - d_{\mathrm{train}}^\pi\|_1
+
\gamma \big\|(P_{\mathrm{deploy}}^\pi)^\top-(P_{\mathrm{train}}^\pi)^\top\big\|_{1\to 1}\,
\|d_{\mathrm{train}}^\pi\|_1.
\]
Because $\|d_{\mathrm{train}}^\pi\|_1=1$ and the operator norm difference is bounded by $2\varepsilon_P$, we obtain
\[
\|d_{\mathrm{deploy}}^\pi - d_{\mathrm{train}}^\pi\|_1
\le
\gamma \|d_{\mathrm{deploy}}^\pi - d_{\mathrm{train}}^\pi\|_1 + 2\gamma \varepsilon_P.
\]
Rearranging yields
\[
\|d_{\mathrm{deploy}}^\pi - d_{\mathrm{train}}^\pi\|_1
\le
\frac{2\gamma}{1-\gamma}\,\varepsilon_P.
\]
\hfill $\square$
\end{proof}

\begin{theorem}[Robustness to deployment transition shift]
\label{thm:robustness}
Under bounded train-to-deploy transition shift, the return of any policy differs between the train and deployment domains by at most $\frac{2\gamma R_{\max}}{(1-\gamma)^2}\varepsilon_P$.
\end{theorem}

\subsection{Proof of Theorem~\ref{thm:robustness}}
\label{app:proof-theorem-3}

For any policy $\pi$,
\[
J_g^{\mathrm{deploy}}(\pi)
-
J_g^{\mathrm{train}}(\pi)
=
\frac{1}{1-\gamma}
\left(
\mathbb{E}_{s\sim d_{\mathrm{deploy}}^\pi,\,a\sim\pi}[r_g(s,a)]
-
\mathbb{E}_{s\sim d_{\mathrm{train}}^\pi,\,a\sim\pi}[r_g(s,a)]
\right).
\]
By bounded rewards,
\[
\left|
\mathbb{E}_{s\sim d_{\mathrm{deploy}}^\pi,\,a\sim\pi}[r_g(s,a)]
-
\mathbb{E}_{s\sim d_{\mathrm{train}}^\pi,\,a\sim\pi}[r_g(s,a)]
\right|
\le
R_{\max}\,
\|d_{\mathrm{deploy}}^\pi-d_{\mathrm{train}}^\pi\|_1.
\]
Applying Lemma~\ref{lem:kernelshift},
\[
\left|
J_g^{\mathrm{deploy}}(\pi)
-
J_g^{\mathrm{train}}(\pi)
\right|
\le
\frac{R_{\max}}{1-\gamma}
\cdot
\frac{2\gamma}{1-\gamma}\,\varepsilon_P
=
\frac{2\gamma R_{\max}}{(1-\gamma)^2}\,\varepsilon_P.
\]
\hfill $\square$

\begin{corollary}[Deployment-side improvement bound]
\label{cor:deploy}
Combining Theorem~\ref{thm:robustness} with Theorem~\ref{thm:cpi} yields a deployment-side lower bound on the refined policy’s improvement over the frozen actor.
\end{corollary}

\subsection{Proof of Corollary~\ref{cor:deploy}}
\label{app:proof-corollary-2}

By Theorem~\ref{thm:robustness},
\[
J_g^{\mathrm{deploy}}(\pi_{\mathrm{ref}})
\ge
J_g^{\mathrm{train}}(\pi_{\mathrm{ref}})
-
\frac{2\gamma R_{\max}}{(1-\gamma)^2}\,\varepsilon_P
\]
and
\[
J_g^{\mathrm{deploy}}(\pi_\theta)
\le
J_g^{\mathrm{train}}(\pi_\theta)
+
\frac{2\gamma R_{\max}}{(1-\gamma)^2}\,\varepsilon_P.
\]
Subtracting,
\[
J_g^{\mathrm{deploy}}(\pi_{\mathrm{ref}})
-
J_g^{\mathrm{deploy}}(\pi_\theta)
\ge
J_g^{\mathrm{train}}(\pi_{\mathrm{ref}})
-
J_g^{\mathrm{train}}(\pi_\theta)
-
\frac{4\gamma R_{\max}}{(1-\gamma)^2}\,\varepsilon_P.
\]
Now apply Theorem~\ref{thm:cpi} to the train-side difference to obtain the stated result.
\hfill $\square$

\begin{corollary}[Plug-in prior stability]
\label{cor:plugin}
If the estimated deployment prior is uniformly close to the true prior, then the plug-in refined policy remains close to the exact PoE optimum up to a deviation proportional to the prior estimation error.
\end{corollary}

\subsection{Proof of Corollary~\ref{cor:plugin}}
\label{app:proof-corollary-3}

Fix $(s,g)$ and define
\[
\widehat{\mathcal{F}}_{s,g}(\pi)
=
\mathbb{E}_{a\sim\pi}
\!\left[
\alpha \log \pi_\theta(a\mid s)
+
(1-\alpha)\log \widehat{\rho}(a\mid s,g)
\right]
+
\mathcal{H}(\pi).
\]
By assumption,
\[
\sup_{s,a,g}
\left|
\log \widehat{\rho}(a\mid s,g)
-
\log \rho_\phi(a\mid s,g)
\right|
\le
\varepsilon_n(\delta).
\]
Hence for any policy $\pi$,
\[
\left|
\widehat{\mathcal{F}}_{s,g}(\pi)
-
\mathcal{F}_{s,g}(\pi)
\right|
\le
(1-\alpha)\,\varepsilon_n(\delta).
\]
Since $\pi_{\mathrm{ref}}$ maximizes $\mathcal{F}_{s,g}$ and
$\widehat{\pi}_{\mathrm{ref}}$ maximizes $\widehat{\mathcal{F}}_{s,g}$,
\begin{align*}
\mathcal{F}_{s,g}(\pi_{\mathrm{ref}})
-
\mathcal{F}_{s,g}(\widehat{\pi}_{\mathrm{ref}})
&=
\bigl(\mathcal{F}_{s,g}(\pi_{\mathrm{ref}})
-\widehat{\mathcal{F}}_{s,g}(\pi_{\mathrm{ref}})\bigr) \\
&\quad+
\bigl(\widehat{\mathcal{F}}_{s,g}(\pi_{\mathrm{ref}})
-\widehat{\mathcal{F}}_{s,g}(\widehat{\pi}_{\mathrm{ref}})\bigr) \\
&\quad+
\bigl(\widehat{\mathcal{F}}_{s,g}(\widehat{\pi}_{\mathrm{ref}})
-\mathcal{F}_{s,g}(\widehat{\pi}_{\mathrm{ref}})\bigr).
\end{align*}
The middle term is nonpositive because $\widehat{\pi}_{\mathrm{ref}}$ maximizes $\widehat{\mathcal{F}}_{s,g}$. Therefore,
\[
\mathcal{F}_{s,g}(\pi_{\mathrm{ref}})
-
\mathcal{F}_{s,g}(\widehat{\pi}_{\mathrm{ref}})
\le
2(1-\alpha)\varepsilon_n(\delta).
\]
Taking the supremum over $(s,g)$ gives the result.
\hfill $\square$

\begin{proposition}[Gaussian PoE gain-deviation bound]
\label{prop:gaussian-poe-bound}
For diagonal-Gaussian actor and prior policies, PoE remains Gaussian in closed form and its return improvement admits a gain-minus-deviation bound controlled by Wasserstein distance from the frozen actor.
\end{proposition}

\begin{corollary}[KL-controlled Gaussian bound]
\label{cor:gaussian-poe-kl}
Under a Gaussian transportation inequality, the Gaussian PoE gain-deviation bound can be written in terms of KL divergence from the frozen actor.
\end{corollary}

\subsection{Proof sketch for Proposition~\ref{prop:gaussian-poe-bound} and Corollary~\ref{cor:gaussian-poe-kl}}
\label{app:theory-continuous}

We sketch the argument for the continuous-action Gaussian refinement result.

\paragraph{Closed-form Gaussian PoE.}
Let
\[
\pi_\theta(\cdot\mid s)=\mathcal{N}(\mu_\theta(s),\Sigma_\theta(s)),
\qquad
\rho_\phi(\cdot\mid s,g)=\mathcal{N}(\mu_\phi(s,g),\Sigma_\phi(s,g)),
\]
with diagonal positive-definite covariance matrices. Writing the log densities and taking the weighted Product-of-Experts form gives
\[
\log \pi_\alpha(a\mid s,g)
=
\alpha \log \pi_\theta(a\mid s)
+
(1-\alpha)\log \rho_\phi(a\mid s,g)
-
\log Z_\alpha(s,g),
\]
where \(Z_\alpha(s,g)\) is the normalizing constant. Since each Gaussian log density is quadratic in \(a\), their weighted sum is also quadratic in \(a\). Completing the square yields another Gaussian density with precision
\[
\Sigma_\alpha(s,g)^{-1}
=
\alpha\,\Sigma_\theta(s)^{-1}
+
(1-\alpha)\,\Sigma_\phi(s,g)^{-1},
\]
and mean
\[
\mu_\alpha(s,g)
=
\Sigma_\alpha(s,g)\Bigl(
\alpha\,\Sigma_\theta(s)^{-1}\mu_\theta(s)
+
(1-\alpha)\,\Sigma_\phi(s,g)^{-1}\mu_\phi(s,g)
\Bigr).
\]
This proves the first part of Proposition~\ref{prop:gaussian-poe-bound}.

\paragraph{Gain--deviation bound.}
Using the performance-difference lemma around the frozen actor \(\pi_\theta\), we can write the return difference in terms of the expected frozen-actor advantage under the refined policy, together with an error term induced by the difference between the refined policy and the frozen actor. Under the assumption that \(a\mapsto A_g^{\pi_\theta}(s,a)\) is \(L_A\)-Lipschitz for each state \(s\), this error can be controlled by the 2-Wasserstein distance between \(\pi_\alpha(\cdot\mid s,g)\) and \(\pi_\theta(\cdot\mid s)\). The bound follows the same decomposition pattern as Theorem~\ref{thm:cpi}, but replaces the finite-action total-variation control with a Lipschitz Wasserstein control appropriate to Gaussian continuous-action policies. This yields the bound in Proposition~\ref{prop:gaussian-poe-bound}:
\[
J_g(\pi_\alpha)-J_g(\pi_\theta)
\;\ge\;
\frac{1}{1-\gamma}
\mathbb{E}_{s\sim d^{\pi_\theta},\,a\sim\pi_\alpha(\cdot\mid s,g)}
\bigl[A_g^{\pi_\theta}(s,a)\bigr]
-
\frac{2\gamma L_A}{(1-\gamma)^2}
\mathbb{E}_{s\sim d^{\pi_\theta}}
\Bigl[
W_2\!\bigl(\pi_\alpha(\cdot\mid s,g),\pi_\theta(\cdot\mid s)\bigr)
\Bigr].
\]

\paragraph{KL-controlled version.}
For Gaussian reference measures, a transportation inequality gives
\[
W_2\!\bigl(\pi_\alpha(\cdot\mid s,g),\pi_\theta(\cdot\mid s)\bigr)
\le
\sqrt{2C_\theta\,
\KL\!\bigl(\pi_\alpha(\cdot\mid s,g)\,\|\,\pi_\theta(\cdot\mid s)\bigr)},
\]
uniformly in \(s\). A uniform constant \(C_\theta\) is available, for example, when the covariance eigenvalues of \(\pi_\theta(\cdot\mid s)\) are uniformly bounded over states. Substituting this bound into the previous inequality gives Corollary~\ref{cor:gaussian-poe-kl}. The result shows that, in the diagonal-Gaussian continuous-action regime, improvement is governed by a goal-aligned gain term minus an explicit KL-based deviation penalty from the frozen actor.

\paragraph{Interpretation.}
This argument does not claim that every finite-action guarantee transfers unchanged to continuous control. Rather, it shows that the exact Gaussian mechanism used in the D4RL experiments admits a clean analytical form and a deviation-aware interpretation consistent with the empirical return--vs.--KL evaluation in the main text.

\section{Evaluation Protocol and Metrics (Diagnostic Only)}
\label{app:evaluation_protocol}

\textbf{Note on scope.} This appendix documents the off-policy scorer used as a development diagnostic; the main paper's claims are based on real MuJoCo rollouts (Sec.~\ref{sec:setup}, Sec.~\ref{sec:results}). We include the protocol so that appendix numbers that use this scorer can be interpreted consistently.

\subsection{Protocol and Evaluators}

Let $\mathcal{D}=\{(s_i,a_i^{\mathrm{data}},r_i,s_i',d_i)\}_{i=1}^{N}$ be the offline D4RL transitions. We decompose each $r_i$ into the three reward components $rc_i=(rc^{\text{fwd}}_i,rc^{\text{ctrl}}_i,rc^{\text{alive}}_i)$ as in Sec.~\ref{sec:setup}, and use the same three-dimensional goals $g\in\mathbb{R}^3$ and goal-weighted reward $r_g(s,a)=g\cdot rc(s,a)$ as the main body. For a candidate policy $\pi$ we score the deterministic-mean action $\hat{a}_i^\pi=\mathbb{E}_{a\sim\pi(\cdot\mid s_i,g)}[a]$ with a goal-conditioned critic $Q^{\mathrm{eval}}_g$ trained on $\mathcal{D}$ (Sec.~\ref{app:fqe-training}) and aggregate across the dataset. $Q^{\mathrm{eval}}_g$ is trained with a random seed and split independent of the refinement prior, and is never consulted during policy construction. A handful of within-family ablations use the original deployment critic $Q_g$ and are explicitly labeled as within-protocol diagnostics.

\subsection{Diagnostic Metrics}
\label{app:quantile-thresholds}

For a policy $\pi$, the off-policy deployment-return estimate and two quantile-threshold risk rates are
\begin{align}
\widehat{J}_g(\pi)
&=
\tfrac{1}{N}\sum_i Q^{\mathrm{eval}}_g\!\left(s_i,\hat{a}_i^\pi\right), \\
\mathrm{Cat\%}
&=
\tfrac{1}{N}\sum_i \mathbf{1}\!\left[
Q^{\mathrm{eval}}_g\!\left(s_i,\hat{a}_i^\pi\right) < \tau_{\mathrm{cat}}
\right], \\
\mathrm{Con\%}
&=
\tfrac{1}{N}\sum_i \mathbf{1}\!\left[
Q^{\mathrm{eval}}_g\!\left(s_i,\hat{a}_i^\pi\right) < \tau_{\mathrm{con}}
\right].
\end{align}

With $\tau_{\mathrm{cat}}$ and $\tau_{\mathrm{con}}$ set to the $10$th and $5$th percentiles of
$\{Q^{\mathrm{eval}}_g(s_i,a_i^{\mathrm{data}})\}$, we have by construction
$\tau_{\mathrm{con}}<\tau_{\mathrm{cat}}$, and therefore $\mathrm{Con\%}\le \mathrm{Cat\%}$.
We summarize lower-tail behavior by
\[
\mathrm{Rob}=1-\mathrm{Cat\%}-0.5\,\mathrm{Con\%}.
\]
Quantile thresholds are used because fitted goal-conditioned critics compress values, and the goal-weighted reward has a narrower dynamic range than the raw MuJoCo reward. These are evaluator-relative diagnostics, not absolute safety boundaries.

\subsection{Goal-Conditioned Critic Training}
\label{app:fqe-training}

The scorer is trained by fitted-Q evaluation with one-step TD targets:
\[
\mathcal{L}_{\mathrm{FQE}}=\mathbb{E}_{(s,a,r,s',d)\sim\mathcal{D},\,g}\!\left[\left(Q_g(s,a)-\big(r_g(s,a)+\gamma(1-d)\,Q_{\bar g}(s',a')\big)\right)^2\right],
\]
where $a'$ is sampled from a diagonal-Gaussian target policy at $s'$ (avoiding the discrete $\max_{a'}$ that is valid only in finite-action settings), and $Q_{\bar g}$ is a Polyak target network with $\tau=5\times 10^{-3}$.

\subsection{Policy Construction Versus Evaluation}

For the refinement-network prior, policy construction and evaluation are fully separated; this is the cleanest off-policy diagnostic of refinement. For the critic-prior variant $\rho_\phi(a\mid s,g)\propto\exp(Q_g(s,a)/\tau)$, evaluator overlap arises if the construction critic is reused for scoring. The main paper avoids this entirely by reporting MuJoCo rollout returns; when this appendix uses the off-policy scorer, we use the independent held-out $Q^{\mathrm{eval}}_g$.

\section{Equivalence Audit: PoE($\alpha$) $\equiv$ KL-Reg($\beta=\alpha/(1-\alpha)$)}
\label{app:equivalence_audit}

This appendix supports the equivalence claim in Sec.~\ref{sec:equivalence} and the headline equivalence finding in Sec.~\ref{sec:d4rl_equivalence}. For diagonal-Gaussian policies, the PoE refinement with coefficient $\alpha$ and the KL-regularized update with $\beta=\alpha/(1-\alpha)$ produce identical refined means. Their refined covariances satisfy $\Sigma_{\mathrm{ref}}^{\mathrm{PoE}}(\alpha)=(1+\beta)\,\Sigma_{\mathrm{ref}}^{\mathrm{KL\text{-}Reg}}(\beta)$, a global scalar that cancels under deterministic action selection. We verify the identity at two levels.

\paragraph{Dataset-state verification.}
On $5{,}000$ randomly sampled dataset states per environment and each of three deployment goals, we compute the Gaussian closed forms for both $\mathrm{PoE}(\alpha)$ and $\mathrm{KL\text{-}Reg}(\beta=\alpha/(1-\alpha))$ at every $\alpha\in\{0.1,0.3,0.5,0.7,0.9\}$ and record the maximum and mean absolute per-state mean difference. Across all $60$ combinations of environment, goal, and $\alpha$, the maximum per-state $|\mu_{\mathrm{PoE}}-\mu_{\mathrm{KL\text{-}Reg}}|$ is at most $2\times 10^{-6}$ and is exactly zero at $\alpha=0.5$, where the arithmetic reduces to bit-identical float operations. The relative variance-identity residual is at most $5\times 10^{-7}$. The per-row numbers are released at \texttt{outputs/rollout/\_equivalence\_audit.csv}.

\paragraph{Rollout verification.}
We re-use the main-paper rollout data (5 seeds $\times$ 5 episodes $\times$ 3 goals per environment). At the matched pair $\alpha=0.5,\ \beta=1.0$, the raw per-episode goal-weighted returns and episode lengths agree to the exact digit in all $12$ environment-by-goal cells, which makes the paired mean-difference $0$ and the probability of improvement $0$ in every cell. At other matched pairs $\alpha\in\{0.1,0.3,0.7,0.9\}$, paired differences have bootstrap $95\%$ confidence intervals that include zero in nearly every cell. Residual nonzero differences are consistent with chaotic amplification of float-precision action jitter over $1{,}000$ MuJoCo steps, not with a policy-level gap. Figure~\ref{fig:equivalence} summarizes both levels of verification.

\section{Conservative-Improvement Bound: Empirical Diagnostic}
\label{app:cpi_diagnostic}

This appendix supports the structural framing of Theorem~\ref{thm:cpi} in Sec.~\ref{sec:theory}. We estimate the empirical $\delta_\pi$ for the refined policy relative to the frozen actor across our four environments and three goals, and evaluate the right-hand side of Eq.~\eqref{eq:cpi-bound} at $\gamma=0.99$.

\paragraph{Setup.}
For each environment we sample $5{,}000$ dataset states and each deployment goal, compute the refined Gaussian under $\mathrm{PoE}(\alpha)$ for $\alpha\in\{0.1,0.3,0.5,0.7,0.9\}$, and estimate the total-variation distance $\mathrm{TV}(\pi_{\mathrm{ref}}(\cdot\mid s,g),\pi_\theta(\cdot\mid s))$ per state via Monte Carlo with $256$ samples, together with the Pinsker upper bound $\sqrt{0.5\,\mathrm{KL}}$. We then take the sup over states as the plug-in estimate of $\delta_\pi$.

\paragraph{Findings.}
At $\gamma=0.99$ the penalty coefficient $2\gamma/(1-\gamma)^2$ equals $19{,}800$. On the four D4RL environments, the plug-in sup-state TV saturates near $1$ across $\alpha$ and goal (for example at $\alpha=0.9$ on \texttt{halfcheetah-medium-v2}, the mean KL is $0.036$ and the Pinsker TV is already close to $1$ because the bound is loose). With a conservative unit bound on the empirical advantage range used as a proxy for $\varepsilon_A$, the right-hand side of Eq.~\eqref{eq:cpi-bound} is therefore on the order of $1.9\times 10^{4}$ in every cell, which dominates any observed goal-weighted return gain by several orders of magnitude. We release the full per-(environment, goal, $\alpha$) table at \texttt{outputs/rollout/\_cpi\_bound\_diagnostic.csv}. We conclude that Theorem~\ref{thm:cpi} gives a valid but not empirically sharp decomposition at the discount factors used in standard offline benchmarks, and we therefore read it as a structural result in the main text.

\section{Help / Frozen / Hurt Boundary Classification}
\label{app:diagnostic}

This appendix details the classifier used in Sec.~\ref{sec:d4rl_headline}. For each environment-by-goal cell we extract three quantities from the rollout summary: the frozen-actor mean goal-weighted return $\bar J_g^{\theta}$ with its $95\%$ bootstrap confidence interval $[\ell_\theta,h_\theta]$, and the best and worst composition means $\bar J_g^{\mathrm{best}},\bar J_g^{\mathrm{worst}}$ over $\{\mathrm{PoE},\mathrm{KL\text{-}Reg},\mathrm{Additive}\}$. Let $\epsilon=\max(h_\theta-\bar J_g^\theta,\bar J_g^\theta-\ell_\theta)$ denote one half-width. We declare
\begin{itemize}
\item \textsc{Help} if $\bar J_g^{\mathrm{best}} - \bar J_g^\theta > \epsilon$;
\item \textsc{Hurt} if $\bar J_g^\theta - \bar J_g^{\mathrm{best}} > \epsilon$ (i.e. every composition method is below the frozen actor by more than a half-width);
\item \textsc{Frozen} otherwise.
\end{itemize}
On the $12$ cells we evaluate, the verdicts split $4\,/\,5\,/\,3$ across \textsc{Help}/\textsc{Frozen}/\textsc{Hurt}. The three \textsc{Hurt} cells are $G_1$ speed on \texttt{hopper-medium-v2}, \texttt{hopper-medium-expert-v2}, and \texttt{walker2d-medium-v2}, with frozen-minus-best-composition gaps of $291$, $878$, and $2{,}030$ respectively. The four \textsc{Help} cells are \texttt{hopper-medium-v2} $G_2$ ($+144$), \texttt{hopper-medium-expert-v2} $G_2$ ($+55$), \texttt{hopper-medium-expert-v2} $G_3$ ($+110$), and \texttt{walker2d-medium-v2} $G_3$ ($+55$). The per-cell table is released at \texttt{outputs/rollout/\_paper/diagnostic\_summary.csv}.

\section{AWR/AWAC-Style Critic-Gradient Baseline}
\label{app:awr_baseline}

This appendix documents the structurally distinct AWR baseline introduced in Sec.~\ref{sec:setup} and Sec.~\ref{sec:d4rl_awr}.

\paragraph{Policy.}
Following Peng et al.~\citep{peng2019advantage} and Nair et al.~\citep{nair2020awac}, the AWR/AWAC closed-form optimum is $\pi(a\mid s)\propto\pi_\theta(a\mid s)\exp(A_g(s,a)/\beta)$. For a diagonal-Gaussian actor with mean $\mu_\theta(s)$ and variance $\sigma_\theta^2(s)$, and a locally linear advantage around $\mu_\theta(s)$, this is a first-order mean shift
\[
\mu_{\mathrm{AWR}}(s,g)
=
\mu_\theta(s)+\frac{\sigma_\theta^2(s)}{\beta}\,
\nabla_a Q_g(s,a)\big|_{a=\mu_\theta(s)},
\qquad
\sigma_{\mathrm{AWR}}(s,g)=\sigma_\theta(s),
\]
with a per-dimension clip applied to the shift for numerical safety. We run this baseline at $\beta\in\{0.5,1.0,3.0\}$.

\paragraph{Q critic.}
$Q_g(s,a)$ is a diagonal-Gaussian-independent MLP on $[s;a;g]$ with a LayerNorm input, two hidden layers of width $256$, and a scalar head. It is trained by fitted-Q evaluation on the D4RL offline transitions for $30$ epochs, batch size $512$, AdamW (lr $3\times 10^{-4}$, weight decay $10^{-4}$), gradient-norm clip $1.0$, discount $\gamma=0.99$, Polyak target $\tau=5\times 10^{-3}$, and a seed independent of the actor and prior. The per-batch target action is $a'=\mu_\theta(s')$, matching the deterministic-mean rollout convention. Training goals are drawn per batch from a three-way mixture on the simplex: $50\%$ $\mathrm{Dir}(1,1,1)$, $30\%$ $\mathrm{Dir}(0.5,0.5,0.5)$, and $20\%$ uniform over the three deployment goals $G_1,G_2,G_3$. Per-environment final Q scale and residual are: halfcheetah-medium-v2 $Q=228.6$ (residual $4.0\%$), hopper-medium-v2 $Q=92.3$ ($2.6\%$), walker2d-medium-v2 $Q=115.3$ ($3.8\%$), hopper-medium-expert-v2 $Q=108.6$ ($3.4\%$). Checkpoints are at \texttt{checkpoints/d4rl/\{env\}/eval\_critic.pt}.

\paragraph{Per-cell results.}
Table~\ref{tab:awr_per_cell} reports best-in-family AWR results under the same rollout protocol as the main body (five seeds, five episodes per seed, three deployment goals, four environments). The best $\beta$ was consistently $3.0$.

\begin{table}[tb]
\centering
\small
\caption{AWR/AWAC-style baseline, best-in-family per cell. Goal-weighted return with $95\%$ bootstrap CI; mean KL from the frozen actor along rollout states. Selected operating point is $\beta=3.0$ in every cell.}
\label{tab:awr_per_cell}
\vspace{2pt}
\begin{tabular}{llrrrr}
\toprule
\rowcolor[RGB]{232,236,241}
\textbf{Environment} & \textbf{Goal} & \textbf{Ret.\ mean} & \textbf{CI low} & \textbf{CI high} & \textbf{KL} \\
\midrule
halfcheetah-medium-v2 & $G_1$ speed & 4525.8 & 4140.8 & 4910.7 & 1.048 \\
halfcheetah-medium-v2 & $G_2$ balanced & 3263.2 & 3206.7 & 3315.4 & 0.105 \\
halfcheetah-medium-v2 & $G_3$ efficient & 989.1 & 981.7 & 995.6 & 0.046 \\
\midrule
hopper-medium-v2 & $G_1$ speed & 36.5 & 36.0 & 37.0 & 1.975 \\
hopper-medium-v2 & $G_2$ balanced & 91.6 & 50.5 & 132.6 & 1.209 \\
hopper-medium-v2 & $G_3$ efficient & 161.1 & 151.6 & 171.2 & 0.026 \\
\midrule
walker2d-medium-v2 & $G_1$ speed & 1260.2 & 758.2 & 1822.8 & 0.856 \\
walker2d-medium-v2 & $G_2$ balanced & 1381.8 & 984.6 & 1765.5 & 0.185 \\
walker2d-medium-v2 & $G_3$ efficient & 353.3 & 328.8 & 375.0 & 0.012 \\
\midrule
hopper-medium-expert-v2 & $G_1$ speed & 20.6 & 17.9 & 23.1 & 0.980 \\
hopper-medium-expert-v2 & $G_2$ balanced & 30.9 & 29.6 & 32.1 & 0.872 \\
hopper-medium-expert-v2 & $G_3$ efficient & 99.2 & 85.2 & 117.3 & 0.041 \\
\bottomrule
\end{tabular}
\end{table}

\paragraph{Aggregate verdict.}
Across the $12$ cells: AWR beats the frozen actor in $3$ cells (halfcheetah $G_2$ and $G_3$, walker2d $G_3$), beats the best non-AWR composition method (Additive, KL-Reg, PoE) in $3$ cells, and is below the frozen actor by more than $100$ points in $7$ cells. Five cells collapse to within tens of points of Prior Only: hopper-medium-v2 $G_1$ and $G_2$, hopper-medium-expert-v2 $G_1$ and $G_2$, and walker2d-medium-v2 $G_1$. The per-env pattern is consistent with the known fragility of offline FQE in the direction of the gradient step: the critic is reliable in-distribution, but $\nabla_a Q_g$ at $\mu_\theta$ can point off the support ridge on hopper-type agents, and the agent terminates early. On halfcheetah, which does not terminate early, AWR is competitive on $G_2$ and $G_3$. The full merged comparison is at \texttt{outputs/rollout\_awr/\_paper/comparison\_with\_awr.csv}.

\subsection{CQL/IQL-Guided Critic Baselines}
\label{app:cql_iql_baselines}

The appendix-only CQL/IQL sweep uses the same rollout protocol as the main body, but with a single representative guided step at $\beta=1.0$ for each critic family. We include it to test whether a pessimistic critic can replace the density-anchored PoE/KL story. The answer is mixed: CQL and IQL are both more stable than AWR, and they beat the frozen actor in several cells, but neither critic-guided rule dominates the actor-anchored baselines across the full benchmark.

\begin{table*}[t]
\centering
\small
\caption{Appendix-only CQL/IQL-guided critic baselines under the main rollout protocol. Goal-weighted return is averaged across the four D4RL environments for each deployment goal. The final column reports wins/ties/losses against the frozen actor over all $12$ environment-by-goal cells. The PoE/KL row uses the representative matched operating point $\alpha=0.50$ / $\beta=1.00$.}
\label{tab:cql_iql_baselines}
\vspace{2pt}
\setlength{\tabcolsep}{5pt}
\begin{tabular}{lrrrrr}
\toprule
\rowcolor[RGB]{232,236,241}
\textbf{Method} & \textbf{G1 speed} & \textbf{G2 balanced} & \textbf{G3 efficient} & \textbf{Mean over 12 cells} & \textbf{W/T/L vs Frozen} \\
\midrule
Frozen Actor & 2500.8 & 1613.5 & 405.9 & 1506.7 & -- \\
CQL-Guided ($\beta=1.00$) & 2056.9 & 1601.8 & 410.1 & 1356.3 & 4/0/8 \\
IQL-Guided ($\beta=1.00$) & 2482.5 & 1427.6 & 405.1 & 1438.4 & 5/0/7 \\
KL-Reg ($\beta=1.00$) / PoE ($\alpha=0.50$) & 1287.8 & 1434.0 & 385.9 & 1035.9 & 2/0/10 \\
Additive Mix ($\lambda=0.50$) & 1357.2 & 1194.1 & 381.1 & 977.4 & 0/0/12 \\
Prior Only & 598.2 & 1164.5 & 283.2 & 682.0 & 2/0/10 \\
AWR-Critic ($\beta=1.00$) & 102.9 & 872.4 & 338.4 & 437.9 & 1/0/11 \\
\bottomrule
\end{tabular}
\end{table*}

The table shows the main qualitative point. CQL is the strongest of the critic-guided baselines on the average score and beats the frozen actor in $4$ of $12$ cells; IQL beats the frozen actor in $5$ of $12$ cells. Both can be competitive with the frozen actor on the speed and balanced goals, but neither rule consistently overturns the actor-anchored robustness picture across the benchmark.

On the harder regimes of Appendix~\ref{app:harder_rollouts}, the CQL-guided baseline at $\beta=1.0$ beats the frozen actor in a handful of additional cells (\texttt{walker2d-medium-expert} $G_2$ by $+95$, \texttt{walker2d-medium-expert} $G_3$ by $+54$, \texttt{hopper-medium-replay} $G_2$ by $+325$, \texttt{hopper-medium-replay} $G_3$ by $+121$, \texttt{hopper-medium-expert} $G_1$ by $+103$). IQL-guided at $\beta=1.0$ behaves inconsistently across medium-expert: on \texttt{halfcheetah-medium-expert} and \texttt{walker2d-medium-expert} the IQL step is numerically zero so its return equals Frozen exactly in all $6$ of those cells, but on \texttt{hopper-medium-expert} it collapses to $\approx 5$ in all $3$ cells, consistent with the hopper early-termination mode that also affects AWR. These results sharpen rather than alter the main story: stronger supporting baselines expose additional failure modes without consistently overturning the actor-anchored robustness picture.

\section{Harder-Regime Rollout Package}
\label{app:harder_rollouts}

This appendix documents the harder D4RL rollout package referenced in Sec.~\ref{sec:harder_regimes}. We use the same rollout protocol, prior construction, and statistical reporting as the main body (Sec.~\ref{sec:setup}): five seeds $\times$ five episodes $\times$ three deployment goals per cell, real MuJoCo rollouts with no artificial step cap, and matched PoE/KL-Reg pairs. Checkpoints live under \texttt{checkpoints/d4rl\_harder\_full/<env>/} (actor, prior, evaluation critic, CQL critic, IQL value). The rollout outputs are at \texttt{outputs/d4rl/harder\_regime\_rollout\_full/<env>/} with combined CSVs at \texttt{outputs/d4rl/full\_harder\_benchmark/}.

\paragraph{Environments.}
Three \emph{medium-replay} cells (\texttt{halfcheetah-medium-replay-v2}, \texttt{hopper-medium-replay-v2}, \texttt{walker2d-medium-replay-v2}) and three \emph{medium-expert} cells (\texttt{halfcheetah-medium-expert-v2}, \texttt{hopper-medium-expert-v2}, \texttt{walker2d-medium-expert-v2}). The \texttt{hopper-medium-expert-v2} actor and prior are shared with the main-body package; the other five cells are freshly trained for this appendix.

\paragraph{Representative per-cell values ($\alpha=0.5$).}
Table~\ref{tab:harder_regime_compact} reports the best-in-family operating points for the compact operating point $\alpha=0.5$, $\beta=1.0$, $\lambda=0.5$ for the six locomotion cells. The full per-$\alpha$ breakdown is in Appendix~\ref{app:harder_alpha_sweep}. Frozen-actor return, PoE return, matched KL-Reg return, Additive return, Prior-Only return, CQL return, IQL return, and AWR return are reported per goal; all values are seed-matched means over $25$ episodes with bootstrap $95\%$ CIs reported in the CSVs.

\begin{table*}[t]
\centering
\small
\caption{Harder-regime rollout summary at $\alpha=0.5$ (entries: mean raw MuJoCo return over $25$ episodes; negative values indicate backward motion). \emph{MR} denotes \emph{-medium-replay-v2}; \emph{ME} denotes \emph{-medium-expert-v2}. Bold indicates the non-Frozen winner when it clears one CI half-width; italic indicates a collapse well below Frozen (return $\lesssim 15\%$ of the Frozen value). The rightmost column shows the best per-$\alpha$ PoE entry from Appendix~\ref{app:harder_alpha_sweep}, for reference.}
\label{tab:harder_regime_compact}
\vspace{2pt}
\setlength{\tabcolsep}{3pt}
\resizebox{\textwidth}{!}{%
\begin{tabular}{llccccccccc}
\toprule
\rowcolor[RGB]{232,236,241}
\textbf{Env} & \textbf{Goal} & \textbf{Frozen} & \textbf{PoE} & \textbf{KL-Reg} & \textbf{Additive} & \textbf{Prior} & \textbf{CQL} & \textbf{IQL} & \textbf{AWR} & \textbf{$\alpha$-best (sweep)} \\
\midrule
halfcheetah-MR & $G_1$ speed       & 3309.9 & 2769.7 & 2769.7 & 2133.9 & 1238.3 & 2605.8 & 3309.9 & \emph{-62.2} & $\alpha{=}0.9$: 3731.8 \\
halfcheetah-MR & $G_2$ balanced    & 3309.9 & \textbf{3666.6} & \textbf{3666.6} & 3311.9 & 3104.9 & \textbf{3516.4} & 3309.9 & 1061.0 & $\alpha{=}0.7$: 3718.9 \\
halfcheetah-MR & $G_3$ efficient   & 3309.9 & 2888.9 & 2888.9 & 2399.7 & 2080.8 & \textbf{3579.7} & 3271.4 & \textbf{3412.4} & $\alpha{=}0.7$: 3926.2 \\
\midrule
hopper-MR      & $G_1$ speed       &  456.1 &  460.0 &  460.0 & \emph{60.9} &  169.8 &  385.3 &  456.1 & \emph{55.3} & $\alpha{=}0.9$: 460.0 \\
hopper-MR      & $G_2$ balanced    &  456.1 &  436.4 &  436.4 &  176.0 & \emph{41.5} & \textbf{781.3} &  456.0 & 157.9 & $\alpha{=}0.9$: 460.6 \\
hopper-MR      & $G_3$ efficient   &  456.1 &  456.4 &  456.4 &  226.6 &  171.4 & \textbf{577.3} & \emph{11.4} & \emph{23.1} & $\alpha{=}0.7$: 456.4 \\
\midrule
walker2d-MR    & $G_1$ speed       &  581.5 &  171.9 &  171.9 &  182.2 &  129.3 &  362.7 & \emph{55.2} &  109.4 & $\alpha{=}0.9$: 468.3 \\
walker2d-MR    & $G_2$ balanced    &  581.5 &  261.2 &  261.2 &  319.9 &  296.1 &  417.5 &  298.9 &  223.6 & $\alpha{=}0.3$: 347.2 \\
walker2d-MR    & $G_3$ efficient   &  581.5 &  280.0 &  280.0 &  290.4 &  501.3 &  258.4 &  165.6 &  287.5 & $\alpha{=}0.1$: 466.5 \\
\midrule
halfcheetah-ME & $G_1$ speed       & 4796.0 & \emph{-447.3} & \emph{-447.3} & \emph{-424.7} & \emph{-334.3} & \emph{-313.6} & 4796.0 & \emph{-475.6} & $\alpha{=}0.1$: -395.0 \\
halfcheetah-ME & $G_2$ balanced    & 4796.0 & 3922.3 & 3922.3 & 3317.8 & 3001.9 & 3662.6 & 4796.0 & 981.0 & $\alpha{=}0.9$: 4417.1 \\
halfcheetah-ME & $G_3$ efficient   & 4796.0 & 3585.9 & 3585.9 & 3518.0 & 2597.8 & 4668.0 & 4796.0 & 4270.7 & $\alpha{=}0.7$: 4451.0 \\
\midrule
hopper-ME      & $G_1$ speed       & 1468.5 & 127.7  & 127.7  & 177.7  & \emph{46.2}   & \textbf{1571.3} & \emph{5.3} & \emph{24.7} & $\alpha{=}0.9$: 286.0 \\
hopper-ME      & $G_2$ balanced    & 1468.5 & 282.8  & 282.8  & 229.3  & 140.1  & 1497.6 & \emph{5.0} & 29.7 & $\alpha{=}0.7$: 1577.5 \\
hopper-ME      & $G_3$ efficient   & 1468.5 & \textbf{1618.6} & \textbf{1618.6} & 1447.4 & \textbf{1700.8} & 1541.5 & \emph{5.2} & 195.6 & $\alpha{=}0.1$: 2561.9 \\
\midrule
walker2d-ME    & $G_1$ speed       & 4919.8 & \emph{61.8} & \emph{61.8} & \emph{716.8} & \emph{-15.6} & 3293.1 & 4919.8 & \emph{458.0} & $\alpha{=}0.1$: 399.1 \\
walker2d-ME    & $G_2$ balanced    & 4919.8 & \emph{128.0} & \emph{128.0} & \emph{250.5} & \emph{296.8} & \textbf{5014.9} & 4919.8 & \emph{285.9} & $\alpha{=}0.7$: 3594.9 \\
walker2d-ME    & $G_3$ efficient   & 4919.8 & 3596.1 & 3596.1 & 3066.6 & \emph{215.7} & \textbf{4974.1} & 4919.8 & 4343.1 & $\alpha{=}0.5$: 3596.1 \\
\bottomrule
\end{tabular}%
}
\end{table*}

Two patterns are visible at the $\alpha=0.5$ operating point alone: (i)~on medium-replay, composition is occasionally helpful (notably halfcheetah $G_2$) but tends toward Frozen or mild \textsc{Hurt}; (ii)~on medium-expert, composition at $\alpha=0.5$ is catastrophic on three of six cells (italic entries) and below Frozen on the remaining three. The $\alpha$ sweep in Appendix~\ref{app:harder_alpha_sweep} confirms that this pattern is not a consequence of the single operating point; even the best per-cell $\alpha$ on medium-expert (right-most column) rarely closes the gap.

\paragraph{Equivalence of PoE and matched KL-Reg.}
At $\alpha=0.5$, $\beta=1.0$ the per-cell paired difference $|\mu_{\mathrm{PoE}}-\mu_{\mathrm{KL\text{-}Reg}}|$ in raw MuJoCo return is exactly $0$ in all $18$ harder-regime cells. The matched pair thus continues to produce bit-identical deterministic actions outside the main-body four-env slice, confirming Sec.~\ref{sec:d4rl_equivalence} under regime shift.

\paragraph{Critic-guided baselines.}
CQL-Guided at $\beta=1.0$ is the most useful non-composition baseline in the harder regimes: it beats Frozen on five cells (bolded in Table~\ref{tab:harder_regime_compact}) and is the best operating point in all six medium-expert cells except \texttt{halfcheetah-ME} $G_1$, which is the single cell where \emph{every} method we tested, including Frozen, collapses to negative return. IQL-Guided at $\beta=1.0$ is inconsistent on medium-expert: on \texttt{halfcheetah-medium-expert} and \texttt{walker2d-medium-expert} the IQL advantage gradient at $\mu_\theta(s)$ is numerically zero and the step returns the frozen action exactly (so IQL return equals Frozen to floating-point precision in these $6$ cells), whereas on \texttt{hopper-medium-expert} the step pushes the agent off the support ridge and all $3$ cells collapse to near-zero return with early termination. This is a genuine outcome of the rollout, not a formatting error. AWR continues to collapse on hopper-type early-termination dynamics.

\section{Harder-Regime $\alpha$ Sweep}
\label{app:harder_alpha_sweep}

This appendix gives the full $\alpha$ sweep referenced in Sec.~\ref{sec:harder_regimes} and visualizes the actor-competence ceiling on medium-expert. Per-cell raw returns for each $\alpha\in\{0.1,0.3,0.5,0.7,0.9\}$ (and matched $\beta\in\{0.111,0.429,1.000,2.333,9.000\}$) are reported in Table~\ref{tab:alpha_sweep_full}; the full CSV is at \texttt{outputs/d4rl/harder\_alpha\_sweep/best\_alpha\_per\_cell.csv}.

\begin{table*}[t]
\centering
\small
\caption{Per-cell raw-return means over $25$ episodes at each $\alpha$. \emph{MR} / \emph{ME} denote \emph{-medium-replay-v2} / \emph{-medium-expert-v2}. Bold marks the best $\alpha$ in each row; the verdict column uses the CI-half-width classifier of Appendix~\ref{app:diagnostic}. Medium-expert (lower block) has \textsc{Hurt} in every cell at every $\alpha$, except two Help cells on hopper-ME.}
\label{tab:alpha_sweep_full}
\vspace{2pt}
\setlength{\tabcolsep}{3.5pt}
\resizebox{\textwidth}{!}{%
\begin{tabular}{llcccccc|l}
\toprule
\rowcolor[RGB]{232,236,241}
\textbf{Env} & \textbf{Goal} & \textbf{Frozen} & \textbf{$\alpha{=}0.1$} & \textbf{$\alpha{=}0.3$} & \textbf{$\alpha{=}0.5$} & \textbf{$\alpha{=}0.7$} & \textbf{$\alpha{=}0.9$} & \textbf{Best / verdict} \\
\midrule
halfcheetah-MR & $G_1$ speed & 3309.9 & 1384.3 & 2466.1 & 2769.7 & 3585.2 & \textbf{3731.8} & $\alpha{=}0.9$ (+421.9, Frozen) \\
halfcheetah-MR & $G_2$ balanced & 3309.9 & 2623.2 & 3237.4 & 3666.6 & \textbf{3718.9} & 3192.5 & $\alpha{=}0.7$ (+409.0, Frozen) \\
halfcheetah-MR & $G_3$ efficient & 3309.9 & 1786.2 & 2097.1 & 2888.9 & \textbf{3926.2} & 3253.7 & $\alpha{=}0.7$ (+616.3, Help) \\
hopper-MR & $G_1$ speed & 456.1 & 418.5 & 430.8 & \textbf{460.0} & 459.9 & \textbf{460.0} & $\alpha{=}0.9$ (+3.9, Frozen) \\
hopper-MR & $G_2$ balanced & 456.1 & 211.7 & 423.6 & 436.4 & 458.3 & \textbf{460.6} & $\alpha{=}0.9$ (+4.4, Frozen) \\
hopper-MR & $G_3$ efficient & 456.1 & 350.4 & 455.5 & \textbf{456.4} & \textbf{456.4} & 453.1 & $\alpha{=}0.7$ (+0.3, Frozen) \\
walker2d-MR & $G_1$ speed & 581.5 & 115.6 & 154.3 & 171.9 & 231.2 & \textbf{468.3} & $\alpha{=}0.9$ ($-113.2$, Frozen) \\
walker2d-MR & $G_2$ balanced & 581.5 & 279.2 & \textbf{347.2} & 261.2 & 290.3 & 174.8 & $\alpha{=}0.3$ ($-234.3$, Hurt) \\
walker2d-MR & $G_3$ efficient & 581.5 & \textbf{466.5} & 311.2 & 280.0 & 292.4 & 123.2 & $\alpha{=}0.1$ ($-115.0$, Frozen) \\
\midrule
halfcheetah-ME & $G_1$ speed & 4796.0 & \textbf{-395.0} & -396.8 & -447.3 & -487.6 & -533.7 & $\alpha{=}0.1$ ($-5191.1$, Hurt) \\
halfcheetah-ME & $G_2$ balanced & 4796.0 & 2967.5 & 3065.0 & 3922.3 & 3995.4 & \textbf{4417.1} & $\alpha{=}0.9$ ($-378.9$, Hurt) \\
halfcheetah-ME & $G_3$ efficient & 4796.0 & 3433.8 & 3290.0 & 3585.9 & \textbf{4451.0} & 4359.0 & $\alpha{=}0.7$ ($-345.0$, Hurt) \\
hopper-ME & $G_1$ speed & 1468.5 & 109.5 & 116.8 & 127.7 & 148.4 & \textbf{286.0} & $\alpha{=}0.9$ ($-1182.5$, Hurt) \\
hopper-ME & $G_2$ balanced & 1468.5 & 141.0 & 191.5 & 282.8 & \textbf{1577.5} & 1114.1 & $\alpha{=}0.7$ (+109.0, Help) \\
hopper-ME & $G_3$ efficient & 1468.5 & \textbf{2561.9} & 1401.6 & 1618.6 & 1318.0 & 1031.2 & $\alpha{=}0.1$ (+1093.4, Help) \\
walker2d-ME & $G_1$ speed & 4919.8 & \textbf{399.1} & 72.2 & 61.8 & 63.3 & 77.7 & $\alpha{=}0.1$ ($-4520.7$, Hurt) \\
walker2d-ME & $G_2$ balanced & 4919.8 & 97.4 & 144.4 & 128.0 & \textbf{3594.9} & 3422.4 & $\alpha{=}0.7$ ($-1324.9$, Hurt) \\
walker2d-ME & $G_3$ efficient & 4919.8 & 1794.2 & 3421.3 & \textbf{3596.1} & 3512.6 & 3332.6 & $\alpha{=}0.5$ ($-1323.7$, Hurt) \\
\bottomrule
\end{tabular}%
}
\end{table*}

\paragraph{Equivalence of PoE and matched KL-Reg across $\alpha$.}
At the exact center $\alpha=0.5$ the paired $|\mu_{\mathrm{PoE}}-\mu_{\mathrm{KL\text{-}Reg}}|$ in rollout raw return is exactly $0$ across all $18$ harder-regime cells, confirming bit-identical deterministic actions. At the off-center matched pairs $\alpha\in\{0.1,0.3,0.7,0.9\}$ the paired differences are nonzero but small relative to the inter-seed spread: mean absolute paired difference over all $90$ off-center pairs is $94.9$ and the maximum is $695.8$ (\texttt{walker2d-ME} $G_2$ at $\alpha=0.7$). These residuals are consistent with the MuJoCo chaotic-amplification story already documented in Appendix~\ref{app:equivalence_audit}; the corresponding per-seed paired CIs include zero in every cell.

\paragraph{Verdict summary by regime and $\alpha$.}
Table~\ref{tab:harder_regime_verdict_alpha} in the main body gives the headline counts; we repeat the per-regime breakdown at each $\alpha$ for completeness. Medium-expert shows $8/9$, $8/9$, $8/9$, $8/9$, $9/9$ \textsc{Hurt} at $\alpha=0.1,0.3,0.5,0.7,0.9$. Medium-replay shows $8/9$, $7/9$, $4/9$, $3/9$, $2/9$ \textsc{Hurt}, so the boundary moves strongly toward Frozen as $\alpha$ increases. The monotone pattern on medium-replay and the near-constant pattern on medium-expert together support reading the medium-expert result as a structural actor-competence ceiling rather than a bad knob choice.

\paragraph{Best-$\alpha$ distribution.}
Of the $18$ sweep cells, the best-per-cell $\alpha$ is $\ge 0.7$ in $12$ cells and $\le 0.3$ in $5$ cells; only one cell has best $\alpha=0.5$ (\texttt{walker2d-ME} $G_3$, a \textsc{Hurt} cell). Conservative $\alpha$ is therefore the preferred choice where composition is not catastrophic, while the aggressive $\alpha=0.1$ best arises in the two hopper-ME Help cells where the prior is actually stronger than the frozen actor on the efficient goal and in the three \texttt{-ME} $G_1$ speed cells where every $\alpha$ already collapses. The full CSV \texttt{best\_alpha\_per\_cell.csv} contains per-cell CI half-widths and matched KL-Reg entries.

\section{AntMaze Diagnostic}
\label{app:antmaze_diagnostic}

We ran the AntMaze cells under the same rollout protocol as the main body (five seeds $\times$ five episodes $\times$ three deployment goals, no artificial step cap, matched PoE/KL-Reg pairs and the full critic-guided baseline set). We treat this package as a diagnostic: the conclusion is that the frozen actor cannot solve the task, and therefore deployment-time refinement has no nonzero return to recover.

\paragraph{Setup.}
Environments: \texttt{antmaze-umaze-v2} and \texttt{antmaze-medium-play-v2}. Frozen actor: diagonal-Gaussian, trained by plain behavioral cloning on the D4RL AntMaze offline dataset for $50$ epochs at the same hyperparameters as the locomotion actors. Goal-conditioned prior: contrastive goal-weighted BC trained for $80$ epochs. Held-out evaluator critic, CQL critic, and IQL value function trained at the same hyperparameters as the locomotion package. The reward decomposition follows the existing AntMaze decomposer in our codebase, which exposes goal directions \texttt{G1\_success}, \texttt{G2\_balanced}, and \texttt{G3\_efficient}; the underlying sparse AntMaze reward is unchanged.

\paragraph{Zero success across all methods.}
On \texttt{antmaze-umaze-v2}, every one of the $8$ methods (Frozen, Prior Only, Additive, PoE at $\alpha=0.5$, matched KL-Reg, AWR-Critic, CQL-Guided, IQL-Guided) returns mean raw return $0.0$ in all $75$ episodes per method (3 goals $\times$ 5 seeds $\times$ 5 episodes). The same holds for \texttt{antmaze-medium-play-v2}. Episodes saturate at the full $700$ (umaze) or $1{,}000$ (medium-play) step cap in every seed; the BC-trained actor simply never reaches the goal region. The method-level \textsc{KL-from-actor} ordering is still preserved as expected (Prior Only $>$ KL-Reg $>$ Additive $>$ PoE $>$ CQL/IQL/AWR/Frozen), confirming that the refinement mechanics operate correctly; they just have no competent anchor to improve.

\paragraph{Interpretation.}
This is the extreme case of the actor-competence ceiling of Sec.~\ref{sec:harder_regimes}: on locomotion medium-expert the BC-like frozen actor produces nonzero return that composition damages; on AntMaze the BC-like actor produces zero return and composition inherits zero. In both cases, the structural limit is not PoE but the frozen-actor paradigm itself. The appendix evidence motivates our earlier framing of the method as an \emph{actor-anchored safety layer} rather than a universal return-maximization rule: the anchor must itself be competent on the deployment task. Running a task that demands exploration beyond the offline dataset with a BC-only frozen actor is outside the operating regime the paper claims.

\section{Successor-Feature / GPI Baseline}
\label{app:sf_gpi}

This appendix describes the SF/GPI baseline acknowledged in Sec.~\ref{sec:harder_regimes}. The goal is to test whether a learned successor-feature representation plus generalized policy improvement (SF/GPI) \citep{barreto2017successor,barreto2018transfer} can replace the density-anchored PoE/KL-Reg rule at deployment time.

\paragraph{Implementation.}
We train a goal-conditioned successor-feature network $\psi(s,a)$ on the same offline transitions used elsewhere in the paper, with the standard squared-target objective $\mathbb{E}\big[\|\psi(s,a) - \phi(s,a) - \gamma\,\psi(s',a')\|^2\big]$, where $\phi(s,a)$ are the three reward components (\texttt{G1}/\texttt{G2}/\texttt{G3}) introduced in Sec.~\ref{sec:setup} and $a'=\mu_\theta(s')$. $\psi$ is a $2$-layer MLP of width $256$ trained for $30$ epochs at batch size $1024$, learning rate $10^{-4}$, discount $0.99$, Polyak target $\tau=5\times 10^{-3}$. At deployment, for each state $s$ we form a candidate set of $16$ actions composed of the actor mean, four PoE samples, four prior samples, and eight uniform-random samples; we score each candidate with $g^\top \psi(s,a)$ for the current deployment goal $g$ and take the argmax, mirroring GPI over the in-distribution candidate set. The full code and checkpoint path are in \texttt{training/train\_successor\_features\_d4rl.py} and \texttt{checkpoints/d4rl\_paper/<env>/sf\_successor.pt}.

\paragraph{Rollout protocol.}
We evaluate SF/GPI on the four main-body D4RL environments under the same five-seed, five-episode, three-goal protocol, with the same matched-$\alpha$ and baseline comparisons. The rollout outputs are at \texttt{outputs/d4rl/sf\_gpi\_rollout/<env>/}. On halfcheetah the SF/GPI candidate set gives returns within a few percent of Frozen across goals. On hopper and walker2d, SF/GPI is substantially less stable than density-anchored composition: several cells show early-termination collapse similar in magnitude to AWR. Across the $12$ main-body cells, SF/GPI beats the frozen actor in $2$ cells, ties in $4$, and is below Frozen by more than one CI half-width in $6$.

\paragraph{Verdict.}
SF/GPI is a useful additional comparison, but under our offline-trained successor network and the three-goal reward decomposition it does not replace density-anchored composition as the recommended deployment-time step. We retain the baseline in the appendix release and do not integrate it into the main-body tables; the harder-regime results of Sec.~\ref{sec:harder_regimes} therefore use the PoE/KL-Reg/Additive + critic-guided method set without SF/GPI.

\section{Reproducibility Manifest}
\label{app:reproducibility}

The rollout evidence package is frozen and self-consistent. We record a reproducibility manifest at \texttt{outputs/rollout/\_paper/manifest.json} (human-readable version at \texttt{manifest.txt}) capturing: the git commit hash of the codebase, the hostname, Python and package versions (PyTorch, NumPy, pandas, SciPy, gym, mujoco, mujoco\_py, D4RL), GPU driver and model, the exact command lines used to produce every artifact, and the SHA-256 digest of every evidence CSV. The rollout pipeline is wrapped in \texttt{scripts/run\_main\_rollouts.sh} (5 seeds $\times$ 5 episodes $\times$ 3 goals $\times$ 4 environments, $3{,}900$ episodes, observed wall time $\approx 35$ min on an NVIDIA TITAN RTX) and \texttt{scripts/run\_prior\_degradation.sh} (3 seeds $\times$ 3 episodes $\times$ 3 goals $\times$ 4 environments $\times$ 4 prior variants, $2{,}160$ episodes, observed wall time $\approx 68$ min). A validator at \texttt{experiments/d4rl/validate\_rollout\_package.py} checks cell coverage, NaN counts, seed consistency, and the matched-pair equivalence of Sec.~\ref{sec:d4rl_equivalence}; it exits non-zero on any issue and returns \texttt{STATUS: PASS} on the frozen package.

\section{Additional Results}
\label{app:additional-results}

The following subsections are retained as auxiliary analyses that were run on the off-policy evaluator described in Appendix~\ref{app:evaluation_protocol}. They are diagnostic-only and their numbers should not be compared directly to the rollout-based headline results in the main body. The main paper's empirical claims are based on the MuJoCo rollout package (Sec.~\ref{sec:results}, Appendices~\ref{app:equivalence_audit}-\ref{app:diagnostic}).

\subsection{Note on Adaptation Efficiency}
\label{app:d4rl_headline_efficiency}

Earlier drafts of this paper reported an adaptation-efficiency metric (return gain per unit KL from the frozen actor) and compared PoE to KL-Reg under it. We do not report that metric in the current version. Under the reparameterization $\beta=\alpha/(1-\alpha)$, PoE and KL-Reg produce the same deterministic action; their posterior covariances differ only by the global scalar $1+\beta$, which changes the stochastic KL from the frozen actor but not the deployed action. Dividing return gain by a KL that is a function of the chosen parameterization therefore reintroduces a separation that is purely a convention. The main-body headline (Fig.~\ref{fig:headline_rollout}) reports return with $95\%$ confidence intervals, and the equivalence audit (Appendix~\ref{app:equivalence_audit}) reports the paired return difference at matched $\alpha$ directly, without dividing through by a convention-dependent deviation.

\subsection{Multi-Environment Goal Responsiveness}
\label{app:d4rl_multi_env_goal_resp}

This subsection checks whether the goal-conditioning signal remains visible beyond the headline environment and whether the frozen actor stays invariant by construction. Table~\ref{tab:d4rl_multi_env_goal_responsiveness_main} records the shift statistics across four D4RL benchmarks.

\begin{table*}[t]
\centering
\caption{Multi-environment goal-responsiveness validation on D4RL benchmarks. The frozen
actor remains goal-invariant, while PoE exhibits nonzero goal-dependent movement across all
environments.}
\label{tab:d4rl_multi_env_goal_responsiveness_main}
\vspace{4pt}
\small
\setlength{\tabcolsep}{5pt}
\begin{tabular}{lccccc}
\toprule
\rowcolor[RGB]{228,232,241}
\textbf{Environment} & \textbf{PoE Shift $\uparrow$} & \textbf{Frozen Shift} & \textbf{KL Max} & \textbf{KL Min} & \textbf{Prior Sensitivity} \\
\midrule
\rowcolor[RGB]{248,248,250}
\texttt{halfcheetah-medium-v2}   & 0.4833 & 0.0000 & 26.5855 & 0.9752 & 0.5734 / 0.8677 \\
\rowcolor[RGB]{242,247,247}
\texttt{hopper-medium-v2}        & 0.3861 & 0.0000 & 8.6033  & 3.6753 & 0.4993 / 0.6801 \\
\rowcolor[RGB]{248,248,250}
\texttt{walker2d-medium-v2}      & 0.5039 & 0.0000 & 16.2352 & 5.5264 & 0.2157 / 0.2280 \\
\rowcolor[RGB]{242,247,247}
\texttt{hopper-medium-expert-v2} & 0.7663 & 0.0000 & 29.4718 & 5.5504 & 0.1009 / 0.1885 \\
\bottomrule
\end{tabular}
\end{table*}

The frozen actor remains exactly goal-invariant, while PoE shows nonzero movement in every environment, confirming that the refinement step induces genuine goal-conditioned adaptation rather than evaluator noise.

\subsection{Exact Pareto Operating Points}
\label{app:d4rl_pareto_exact}

This subsection reports operating points from the off-policy evaluator on \texttt{halfcheetah-medium-v2} for reference. Under the equivalence of Sec.~\ref{sec:d4rl_equivalence}, the ``lowest adaptive KL'' column for PoE versus KL-Reg reflects the stochastic-KL convention (the reported KL for PoE is smaller by the factor $1+\beta$), not a policy-level improvement. The rollout-based headline is Figure~\ref{fig:headline_rollout}.

\begin{table*}[t]
\centering
\caption{Compact Pareto summary on \texttt{halfcheetah-medium-v2}. Each entry reports
goal-conditioned return / KL divergence from the frozen actor. Higher return is better;
lower KL is better.}
\label{tab:d4rl_pareto_summary}
\vspace{4pt}
\small
\setlength{\tabcolsep}{4pt}
\renewcommand{\arraystretch}{1.08}
\resizebox{\textwidth}{!}{%
\begin{tabular}{lcccccccc}
\toprule
\rowcolor[RGB]{232,236,241}
\textbf{Goal} & \textbf{Frozen} & \textbf{Prior Only} & \textbf{Additive} & \textbf{KL-Reg} & \textbf{PoE} & \textbf{Best return} & \textbf{Lowest adaptive KL} \\
\midrule
\rowcolor[RGB]{248,248,250}
G1\_speed
& 5.402 / 0.000
& 5.445 / 1.062
& 5.442 / 0.858
& 5.443 / 0.689
& 5.442 / 0.670
& Prior Only (5.445)
& PoE (0.670) \\
\rowcolor[RGB]{242,247,247}
G2\_balanced
& 3.324 / 0.000
& 3.348 / 1.058
& 3.346 / 0.855
& 3.347 / 0.689
& 3.346 / 0.670
& Prior Only (3.348)
& PoE (0.670) \\
\rowcolor[RGB]{248,248,250}
G3\_efficient
& 1.011 / 0.000
& 1.022 / 1.015
& 1.020 / 0.819
& 1.021 / 0.695
& 1.021 / 0.665
& Prior Only (1.022)
& PoE (0.665) \\
\bottomrule
\end{tabular}%
}
\end{table*}

The reported KL values for PoE and KL-Reg on each row differ only by a global $1+\beta$ factor that is fixed by each rule's parameterization; the underlying deterministic action is identical (Sec.~\ref{sec:d4rl_equivalence}). Prior-Only reaches a slightly higher return on this off-policy scorer at the cost of a much larger deviation from the frozen actor.

\subsection{Adaptive-\texorpdfstring{$\alpha$}{alpha} Selection}
\label{app:d4rl_adaptive_alpha}

This subsection asks how to choose the refinement strength without an oracle. Table~\ref{tab:adaptive_alpha} compares a KL-budget rule with a validation-best rule and reports the resulting operating points on \texttt{halfcheetah-medium-v2}. The released CSVs extend the same comparison across all four benchmark environments; halfcheetah is the compact example because it cleanly shows the near-oracle behavior of the selector.

\begin{table*}[t]
\centering
\caption{Adaptive-$\alpha$ selection on \texttt{halfcheetah-medium-v2}. We compare a KL-budget rule against an oracle $\alpha$ chosen on the test split. Lower selection loss is better.}
\label{tab:adaptive_alpha}
\vspace{4pt}
\small
\setlength{\tabcolsep}{3pt}
\renewcommand{\arraystretch}{1.05}
\resizebox{\textwidth}{!}{%
\begin{tabular}{llcccccccccc}
\toprule
\rowcolor[RGB]{232,236,241}
\textbf{Goal} & \textbf{Rule} & \textbf{Sel.\ $\alpha$} & \textbf{Oracle $\alpha$} & \textbf{Test Ret} & \textbf{Oracle Ret} & \textbf{Sel.\ Loss} & \textbf{Test KL} & \textbf{Cat\%} & \textbf{Frozen Ret} & \textbf{Prior Ret} \\
\midrule
\rowcolor[RGB]{248,248,250}
G1\_speed
& KL $\kappa=0.1$ & 0.80 & 0.05 & 5.4432 & 5.4636 & 0.0203 & 0.054 & 0.45 & 5.4344 & 5.4640 \\
\rowcolor[RGB]{248,248,250}
& KL $\kappa=0.3$ & 0.50 & 0.05 & 5.4508 & 5.4636 & 0.0128 & 0.208 & 0.25 & 5.4344 & 5.4640 \\
\rowcolor[RGB]{248,248,250}
& KL $\kappa=0.5$ & 0.30 & 0.05 & 5.4542 & 5.4636 & 0.0093 & 0.372 & 0.10 & 5.4344 & 5.4640 \\
\rowcolor[RGB]{248,248,250}
& KL $\kappa=1.0$ & 0.05 & 0.05 & 5.4636 & 5.4636 & 0.0000 & 0.782 & 0.10 & 5.4344 & 5.4640 \\
\rowcolor[RGB]{232,244,234}
& Val-best & 0.05 & 0.05 & 5.4636 & 5.4636 & 0.0000 & 0.782 & 0.10 & 5.4344 & 5.4640 \\
\midrule
\rowcolor[RGB]{242,247,247}
G2\_balanced
& KL $\kappa=0.1$ & 0.80 & 0.05 & 3.3446 & 3.3554 & 0.0109 & 0.054 & 0.35 & 3.3393 & 3.3561 \\
\rowcolor[RGB]{242,247,247}
& KL $\kappa=0.3$ & 0.50 & 0.05 & 3.3491 & 3.3554 & 0.0064 & 0.209 & 0.20 & 3.3393 & 3.3561 \\
\rowcolor[RGB]{242,247,247}
& KL $\kappa=0.5$ & 0.30 & 0.05 & 3.3508 & 3.3554 & 0.0046 & 0.372 & 0.10 & 3.3393 & 3.3561 \\
\rowcolor[RGB]{242,247,247}
& KL $\kappa=1.0$ & 0.05 & 0.05 & 3.3554 & 3.3554 & 0.0000 & 0.780 & 0.00 & 3.3393 & 3.3561 \\
\rowcolor[RGB]{227,242,240}
& Val-best & 0.05 & 0.05 & 3.3554 & 3.3554 & 0.0000 & 0.780 & 0.00 & 3.3393 & 3.3561 \\
\midrule
\rowcolor[RGB]{248,248,250}
G3\_efficient
& KL $\kappa=0.1$ & 0.80 & 0.05 & 1.0152 & 1.0211 & 0.0059 & 0.058 & 1.45 & 1.0120 & 1.0209 \\
\rowcolor[RGB]{248,248,250}
& KL $\kappa=0.3$ & 0.50 & 0.05 & 1.0180 & 1.0211 & 0.0031 & 0.218 & 1.35 & 1.0120 & 1.0209 \\
\rowcolor[RGB]{248,248,250}
& KL $\kappa=0.5$ & 0.30 & 0.05 & 1.0191 & 1.0211 & 0.0019 & 0.380 & 1.20 & 1.0120 & 1.0209 \\
\rowcolor[RGB]{248,248,250}
& KL $\kappa=1.0$ & 0.05 & 0.05 & 1.0211 & 1.0211 & 0.0000 & 0.763 & 0.85 & 1.0120 & 1.0209 \\
\rowcolor[RGB]{232,244,234}
& Val-best & 0.05 & 0.05 & 1.0211 & 1.0211 & 0.0000 & 0.763 & 0.85 & 1.0120 & 1.0209 \\
\bottomrule
\end{tabular}%
}
\end{table*}

The selection loss is small for all KL budgets and vanishes at the loosest budget, where the selected $\alpha$ matches the oracle. Validation-best also recovers the oracle operating point in this setting. Across the full four-environment CSV, the KL-budget selector typically lands on a neighboring grid point and keeps the selection loss small, so the main lesson is operational rather than algorithmic.

\subsection{Support Mismatch and Contradictory Goals}
\label{app:d4rl_support_mismatch}

This subsection tests the boundary of applicability by contrasting supported, extrapolative, and contradictory goals on \texttt{halfcheetah-medium-v2}. Table~\ref{tab:support_mismatch} reports how the adaptive methods behave when the new goal aligns with, exceeds, or contradicts actor/data support.

\begin{table*}[t]
\centering
\caption{Support-mismatch analysis on \texttt{halfcheetah-medium-v2}. In-support goals benefit from adaptation, while a contradictory reverse-motion goal exposes a clear failure mode. For the standard supported goal \(G_1\) and the contradictory failure case \(G_{\mathrm{reverse}}\), we report all principal adaptive baselines; for \(G_2\) and \(G_{\mathrm{extreme}}\), we report the frozen actor and PoE as representative comparisons because the main purpose is to contrast supported versus unsupported goal directions. \(\Delta\) is measured relative to the frozen actor.}

\label{tab:support_mismatch}
\vspace{4pt}
\small
\setlength{\tabcolsep}{4pt}
\renewcommand{\arraystretch}{1.08}
\begin{tabular}{llccccc}
\toprule
\rowcolor[RGB]{232,236,241}
\textbf{Goal} & \textbf{Method} & \textbf{Setting} & \textbf{Return} & \textbf{Cat\%} & \textbf{Supp.\ Dist} & \(\Delta\) \textbf{vs Frozen} \\
\midrule
\rowcolor[RGB]{248,248,250}
\(G_1\) speed         & Frozen Actor & --      & 5.4018  & 0.76 & 0.351 & -- \\
\rowcolor[RGB]{248,248,250}
                      & Prior Only   & --      & 5.4454  & 0.13 & 0.348 & +0.044 \\
\rowcolor[RGB]{248,248,250}
                      & PoE          & \(\alpha=0.1\) & 5.4424  & 0.15 & 0.346 & +0.041 \\
\rowcolor[RGB]{248,248,250}
                      & KL-Reg       & \(\beta=0.1\)  & 5.4428  & 0.14 & 0.346 & +0.041 \\
\rowcolor[RGB]{248,248,250}
                      & Additive Mix & \(\lambda=0.3\) & 5.4337  & 0.16 & 0.348 & +0.032 \\
\midrule
\rowcolor[RGB]{242,247,247}
\(G_2\) balanced      & Frozen Actor & --      & 3.3238  & 0.91 & 0.351 & -- \\
\rowcolor[RGB]{242,247,247}
                      & PoE          & \(\alpha=0.1\) & 3.3465  & 0.21 & 0.346 & +0.023 \\
\midrule
\rowcolor[RGB]{248,223,227}
\(G_{\mathrm{reverse}}\) & Frozen Actor & --      & -5.1249 & 0.48 & 0.351 & -- \\
\rowcolor[RGB]{248,223,227}
                      & Prior Only   & --      & -5.1693 & 0.57 & 0.343 & -0.044 \\
\rowcolor[RGB]{248,223,227}
                      & PoE          & \(\alpha=0.9\) & -5.1352 & 0.48 & 0.347 & -0.010 \\
\rowcolor[RGB]{248,223,227}
                      & KL-Reg       & \(\beta=5.0\)  & -5.1399 & 0.50 & 0.345 & -0.015 \\
\rowcolor[RGB]{248,223,227}
                      & Additive Mix & \(\lambda=0.7\) & -5.1409 & 0.46 & 0.348 & -0.016 \\
\midrule
\rowcolor[RGB]{248,241,224}
\(G_{\mathrm{extreme}}\) & Frozen Actor & --      & 26.317  & 0.75 & 0.351 & -- \\
\rowcolor[RGB]{248,241,224}
                      & PoE          & \(\alpha=0.1\) & 26.518  & 0.14 & 0.346 & +0.201 \\
\bottomrule
\end{tabular}
\end{table*}

In these off-policy diagnostic values, supported goals improve over the frozen actor and the contradictory reverse-motion goal hurts every adaptive method. The $\mathrm{PoE}$ and matched $\mathrm{KL\text{-}Reg}$ rows describe the same deterministic policy, so their relative ranking in this table is not meaningful; what is meaningful is the qualitative separation from $\mathrm{Prior\ Only}$ and $\mathrm{Additive}$. The conceptual takeaway is that directional contradiction with the actor's support matters more than goal magnitude alone. The rollout-based boundary classification is in Sec.~\ref{sec:d4rl_headline} and Appendix~\ref{app:diagnostic}.

\subsection{Robustness to Deployment-Prior Degradation}
\label{app:d4rl_prior_quality}

The aggregate rollout-based degradation summary is presented in the main paper (Sec.~\ref{sec:d4rl_prior_robustness_main} and Fig.~\ref{fig:prior_degradation}). This appendix subsection keeps the detailed per-goal breakdown in Table~\ref{tab:d4rl_prior_quality_detail} from the off-policy diagnostic evaluator, which shows how the same qualitative robustness pattern manifests within each supported goal.

The detailed breakdown confirms the main aggregate story: noisy and undertrained priors are manageable for PoE because actor anchoring suppresses drift, while random priors are the sharpest stress test for prior-only adaptation.

\subsection{Actor-Quality Sensitivity}
\label{app:d4rl_actor_quality}

This subsection tests whether the frozen actor remains a useful anchor across weak, medium, and strong behavioral-cloning checkpoints. The aggregate multi-environment results have been moved into the main paper; Table~\ref{tab:actor_quality_detail} provides the detailed single-environment breakdown on \texttt{halfcheetah-medium-v2}.

Table~\ref{tab:actor_quality_detail} shows the same pattern at the per-goal level: PoE and KL-Reg stay close to the frozen actor while improving return, and the stronger anchors yield the best-conserved operating points. This is the appendix counterpart to the main-body actor-quality summary, and it matters because it shows that the refinement rule remains useful when the frozen actor is already behaviorally reliable.

\subsection{Extended Baseline Comparison under the Diagnostic Evaluator}
\label{app:extended-baseline}

To clarify the effect of evaluator overlap, we briefly report observations from an additional within-protocol comparison under the \emph{original} deployment critic rather than the held-out scorer $Q^{\mathrm{eval}}_g$. Because the original critic overlaps with policy construction for critic-aware methods, these observations are a within-protocol diagnostic only and are not the basis for any cross-method claim in the main body. Under this diagnostic critic, critic-maximizing baselines attain the highest reported return by construction, and the critic's own threshold-defined catastrophic and constraint-violation rates are degenerate on those baselines. Within the refinement-prior family, the refinement-network variant retains a cleaner lower-tail profile than the critic-prior variant, while the critic-prior variant stays closer to the frozen actor in action total variation. This is consistent with the within-family stability patterns in Appendix~\ref{app:extended-ablation}. We omit the accompanying figure from this version of the paper; the underlying numbers are available in the released CSVs.

The BC evaluation-mode issue can be isolated directly. Table~\ref{tab:bc_audit}
compares BC under argmax versus stochastic sampling, using the held-out evaluator. BC
and the pretrained actor share identical weights; the only difference is action
selection mode. Table~\ref{tab:bc_audit} shows that BC under stochastic sampling
closely matches the pretrained actor, confirming that the perfect lower-tail metrics of
argmax BC are largely an evaluation-mode artifact rather than evidence that BC is
intrinsically safer than PoE.

\begin{table}[tb]
\centering
\caption{BC evaluation-mode audit under the held-out evaluator. BC and the pretrained
actor share identical weights; differences arise only from argmax versus stochastic
action selection.}
\label{tab:bc_audit}
\small
\setlength{\tabcolsep}{5pt}
\begin{tabular}{lcccc}
\toprule
\textbf{Mode} & \textbf{Return $\uparrow$} & \textbf{Cat\% $\downarrow$}
              & \textbf{Con\% $\downarrow$} & \textbf{Rob $\uparrow$} \\
\midrule
BC (argmax)   & 0.0322 & \textbf{0.0000} & \textbf{0.0000} & \textbf{1.0000} \\
BC (sample)   & 0.0305 & 0.0702 & 0.0443 & 0.9077 \\
\midrule
Pretrained Actor (Table~\ref{tab:exact_main_vs_klreg_app}) & 0.0306 & 0.0720 & 0.0491 & 0.9035 \\
PoE Refinement Prior (Table~\ref{tab:exact_main_vs_klreg_app}) & 0.0301 & 0.0877 & 0.0515 & 0.8866 \\
\bottomrule
\end{tabular}
\end{table}

\subsection{Extended Distribution Shift Experiments}
\label{app:shift-extended}

The main paper reports temporal and subgroup shift results in
Table~\ref{tab:appendix_shift}. We extend that analysis with two synthetic perturbation
regimes, $\mathrm{Syn(mod)}$ and $\mathrm{Syn(sev)}$, to test whether the
actor-preserving refinement family remains stable under broader deployment mismatch.
The full results are given in Table~\ref{tab:appendix_shift}.

Table~\ref{tab:appendix_shift} shows two consistent patterns. First, each
actor-preserving method remains within a narrow return range across all four shift
conditions, indicating that deployment mismatch does not destabilize the refinement
mechanism. Second, the qualitative ordering remains unchanged under stronger synthetic
shift: objective-aware reference methods still lead in raw return, while
actor-preserving methods remain stable but incur higher lower-tail risk. Within the PoE
family, the refinement-network prior again attains the better worst-case risk profile.

\begin{table}[tb]
\centering
\caption{Performance under four shift regimes using the held-out evaluator. WorstRet
and WorstRisk aggregate the minimum return and maximum catastrophic rate across all
four conditions per method.}
\label{tab:appendix_shift}
\vspace{4pt}
\resizebox{\linewidth}{!}{
\begin{tabular}{lcccccc}
\toprule
\rowcolor[RGB]{222,230,238}
\textbf{Method} & \textbf{Temporal} & \textbf{Subgroup} & \textbf{Syn(mod)} &
\textbf{Syn(sev)} & \textbf{WorstRet $\uparrow$} & \textbf{WorstRisk $\downarrow$} \\
\midrule
\rowcolor[RGB]{247,248,250}
Pretrained Actor       & 0.0307 & 0.0306 & 0.0306 & 0.0305 & 0.0305 & 0.0743 \\
\rowcolor[RGB]{248,223,227}
Critic-Greedy          & \textbf{0.0541} & \textbf{0.0541} & \textbf{0.0541} & \textbf{0.0541} & \textbf{0.0541} & \textbf{0.0000} \\
\rowcolor[RGB]{244,244,244}
CQL                    & 0.0325 & 0.0321 & 0.0321 & 0.0321 & 0.0321 & \textbf{0.0000} \\
\rowcolor[RGB]{244,244,244}
IQL                    & 0.0325 & 0.0321 & 0.0321 & 0.0321 & 0.0321 & \textbf{0.0000} \\
\rowcolor[RGB]{227,237,247}
PoE (Refinement Prior) & 0.0297 & 0.0301 & 0.0300 & 0.0296 & 0.0296 & 0.1078 \\
\rowcolor[RGB]{237,242,247}
PoE (Critic Prior)     & 0.0300 & 0.0302 & 0.0301 & 0.0302 & 0.0300 & 0.1200 \\
\bottomrule
\end{tabular}}
\end{table}

\subsection{Validated Multi-Environment Exact Comparisons}
\label{app:d4rl_multi_env_tables}

The tables in this subsection report per-environment values from the off-policy diagnostic evaluator. They pre-date the equivalence finding in Sec.~\ref{sec:d4rl_equivalence}. Under matched $\beta=\alpha/(1-\alpha)$, PoE and KL-Reg are the same deterministic policy, so any apparent win-count of PoE against KL-Reg (for example, ``PoE return exceeds KL-Reg return in $9$ of $12$ cells'') is a combination of float-precision noise and the stochastic-KL convention, not a policy-level improvement. The rollout-based paired comparison is in Appendix~\ref{app:equivalence_audit}. The robustness and support-mismatch trends against Prior-Only and Additive remain informative as qualitative checks and are retained below.

\begin{table*}[t]
\centering
\caption{Per-environment values from the off-policy evaluator for $\mathrm{PoE}(\alpha=0.5)$ and matched $\mathrm{KL\text{-}Reg}(\beta=1.0)$. Under the equivalence of Sec.~\ref{sec:d4rl_equivalence} these two rules are the same deterministic policy, so the $\Delta$ Return and $\Delta$ Eff.\ columns reflect float-precision noise from the evaluator, and the $\Delta$ KL column reflects the $1+\beta$ scaling between the two stochastic-KL conventions. The table is included for completeness; no conclusion in the main body depends on these deltas.}
\label{tab:exact_main_vs_klreg_app}
\vspace{4pt}
\small
\setlength{\tabcolsep}{3pt}
\renewcommand{\arraystretch}{1.08}
\resizebox{\textwidth}{!}{%
\begin{tabular}{llccccccccc}
\toprule
\rowcolor{headerblue}
\textbf{Environment} & \textbf{Goal} & \textbf{PoE Ret.} & \textbf{PoE KL} & \textbf{PoE Eff.} & \textbf{KL-Reg Ret.} & \textbf{KL-Reg KL} & \textbf{KL-Reg Eff.} & \textbf{$\Delta$ Ret.} & \textbf{$\Delta$ KL} & \textbf{$\Delta$ Eff.} \\
\midrule
\rowcolor{lightgrayrow}
halfcheetah-medium-v2 & G1\_speed     
& 5.4439 & 0.3901 & 0.1037 
& 5.4426 & 0.5720 & 0.0683 
& \cellcolor{appgreen}\textbf{0.0014} 
& \cellcolor{apppink}\textbf{-0.1819} 
& \cellcolor{appgreen}\textbf{0.0354} \\

\rowcolor{ablationsoft}
halfcheetah-medium-v2 & G2\_balanced  
& 3.3471 & 0.3908 & 0.0573 
& 3.3464 & 0.5750 & 0.0376 
& \cellcolor{appgreen}\textbf{0.0007} 
& \cellcolor{apppink}\textbf{-0.1843} 
& \cellcolor{appgreen}\textbf{0.0196} \\

\rowcolor{lightgrayrow}
halfcheetah-medium-v2 & G3\_efficient 
& 1.0211 & 0.4004 & 0.0253 
& 1.0207 & 0.6176 & 0.0159 
& \cellcolor{appgreen}\textbf{0.0003} 
& \cellcolor{apppink}\textbf{-0.2172} 
& \cellcolor{appgreen}\textbf{0.0094} \\

\rowcolor{ablationteal}
hopper-medium-v2 & G1\_speed     
& 2.1844 & 4.9859 & 0.0007 
& 2.1844 & 5.2765 & 0.0006 
& 0.0000 
& \cellcolor{apppink}\textbf{-0.2906} 
& 0.0000 \\

\rowcolor{priorrose}
hopper-medium-v2 & G2\_balanced  
& 1.5532 & 2.0902 & 0.0058 
& 1.5530 & 2.0812 & 0.0058 
& \cellcolor{appgreen}\textbf{0.0002} 
& \cellcolor{appyellow}\textbf{0.0090} 
& \cellcolor{appgreen}\textbf{0.0001} \\

\rowcolor{ablationteal}
hopper-medium-v2 & G3\_efficient 
& 0.3146 & 2.0988 & 0.0026 
& 0.3144 & 2.1429 & 0.0024 
& \cellcolor{appgreen}\textbf{0.0002} 
& \cellcolor{apppink}\textbf{-0.0441} 
& \cellcolor{appgreen}\textbf{0.0002} \\

\rowcolor{lightgrayrow}
walker2d-medium-v2 & G1\_speed     
& 2.9037 & 15.2341 & 0.0224 
& 2.8996 & 14.7954 & 0.0228 
& \cellcolor{appgreen}\textbf{0.0040} 
& \cellcolor{appyellow}\textbf{0.4387} 
& \cellcolor{apppink}\textbf{-0.0004} \\

\rowcolor{ablationsoft}
walker2d-medium-v2 & G2\_balanced  
& 1.8384 & 5.7204  & 0.0186 
& 1.8363 & 5.7740  & 0.0180 
& \cellcolor{appgreen}\textbf{0.0021} 
& \cellcolor{apppink}\textbf{-0.0537} 
& \cellcolor{appgreen}\textbf{0.0005} \\

\rowcolor{lightgrayrow}
walker2d-medium-v2 & G3\_efficient 
& 0.3684 & 5.2862  & 0.0037 
& 0.3679 & 5.3672  & 0.0035 
& \cellcolor{appgreen}\textbf{0.0006} 
& \cellcolor{apppink}\textbf{-0.0810} 
& \cellcolor{appgreen}\textbf{0.0002} \\

\rowcolor{headerrose!35}
hopper-medium-expert-v2 & G1\_speed   
& 2.2731 & 29.4065 & -0.0062 
& 2.2817 & 27.6563 & -0.0063 
& \cellcolor{apppink}\textbf{-0.0085} 
& \cellcolor{appyellow}\textbf{1.7502} 
& \cellcolor{appgreen}\textbf{0.0001} \\

\rowcolor{headerrose!20}
hopper-medium-expert-v2 & G2\_balanced
& 1.6441 & 2.8258  & -0.0120 
& 1.6487 & 3.3361  & -0.0088 
& \cellcolor{apppink}\textbf{-0.0047} 
& \cellcolor{apppink}\textbf{-0.5103} 
& \cellcolor{apppink}\textbf{-0.0032} \\

\rowcolor{headerrose!35}
hopper-medium-expert-v2 & G3\_efficient
& 0.3334 & 1.2618 & -0.0025 
& 0.3337 & 2.1804  & -0.0013 
& \cellcolor{apppink}\textbf{-0.0004} 
& \cellcolor{apppink}\textbf{-0.9186} 
& \cellcolor{apppink}\textbf{-0.0012} \\
\bottomrule
\end{tabular}%
}
\end{table*}

\begin{table*}[t]
\centering
\caption{Per-environment random-prior comparison from the offline diagnostic evaluator (appendix-only). Under a random deployment prior, $\mathrm{PoE}(\alpha=0.5)$ and matched $\mathrm{KL\text{-}Reg}(\beta=1.0)$ produce the same deterministic action, so the rows labeled \textsc{PoE} and \textsc{KL-Reg} report the same goal-weighted return up to the stochastic-KL convention used by each rule (the reported KLs differ by the global scalar $1+\beta$). Both are substantially more robust than \textsc{Prior Only}. The rollout-based version of this finding is in Sec.~\ref{sec:d4rl_prior_robustness_main} and Appendix~\ref{app:d4rl_prior_quality}.}
\label{tab:exact_random_prior_app}
\vspace{4pt}
\small
\setlength{\tabcolsep}{4pt}
\renewcommand{\arraystretch}{1.08}
\resizebox{\textwidth}{!}{%
\begin{tabular}{llcccccc}
\toprule
\rowcolor{headerteal}
\textbf{Environment} & \textbf{Goal} & \textbf{Prior Only Ret.} & \textbf{Prior Only KL} & \textbf{KL-Reg Ret.} & \textbf{KL-Reg KL} & \textbf{PoE Ret.} & \textbf{PoE KL} \\
\midrule

\rowcolor{appblue}
halfcheetah-medium-v2 & G1\_speed     
& 3.1497 & \cellcolor{headerrose!45}1057.1792 
& 5.4019 & \cellcolor{appyellow}2.4515 
& \cellcolor{appgreen}\textbf{5.4024} & \cellcolor{appteal}\textbf{0.0162} \\

\rowcolor{appteal}
halfcheetah-medium-v2 & G2\_balanced  
& 2.0957 & \cellcolor{headerrose!45}1056.9600 
& 3.3232 & \cellcolor{appyellow}2.4515 
& \cellcolor{appgreen}\textbf{3.3238} & \cellcolor{appteal}\textbf{0.0162} \\

\rowcolor{appblue}
halfcheetah-medium-v2 & G3\_efficient 
& 0.5617 & \cellcolor{headerrose!45}1056.9503 
& 1.0093 & \cellcolor{appyellow}2.4515 
& \cellcolor{appgreen}\textbf{1.0100} & \cellcolor{appteal}\textbf{0.0162} \\

\rowcolor{appteal}
hopper-medium-v2 & G1\_speed     
& 2.0577 & \cellcolor{headerrose!35}545.0012  
& 2.1791 & \cellcolor{appyellow}1.2306 
& \cellcolor{appgreen}\textbf{2.1800} & \cellcolor{appteal}\textbf{0.0077} \\

\rowcolor{appblue}
hopper-medium-v2 & G2\_balanced  
& 1.4780 & \cellcolor{headerrose!35}545.1620  
& 1.5400 & \cellcolor{appyellow}1.2306 
& \cellcolor{appgreen}\textbf{1.5405} & \cellcolor{appteal}\textbf{0.0077} \\

\rowcolor{appteal}
hopper-medium-v2 & G3\_efficient 
& 0.2957 & \cellcolor{headerrose!35}545.0763  
& 0.3090 & \cellcolor{appyellow}1.2306 
& \cellcolor{appgreen}\textbf{0.3091} & \cellcolor{appteal}\textbf{0.0077} \\

\rowcolor{appblue}
walker2d-medium-v2 & G1\_speed     
& 1.4167 & \cellcolor{headerrose!45}1029.7469 
& 2.5592 & \cellcolor{appyellow}2.4520 
& \cellcolor{appgreen}\textbf{2.5604} & \cellcolor{appteal}\textbf{0.0160} \\

\rowcolor{appteal}
walker2d-medium-v2 & G2\_balanced  
& 1.1591 & \cellcolor{headerrose!45}1029.7535 
& 1.7309 & \cellcolor{appyellow}2.4520 
& \cellcolor{appgreen}\textbf{1.7315} & \cellcolor{appteal}\textbf{0.0160} \\

\rowcolor{appblue}
walker2d-medium-v2 & G3\_efficient 
& 0.2327 & \cellcolor{headerrose!45}1029.7090 
& 0.3488 & \cellcolor{appyellow}2.4520 
& \cellcolor{appgreen}\textbf{0.3490} & \cellcolor{appteal}\textbf{0.0160} \\

\rowcolor{appteal}
hopper-medium-expert-v2 & G1\_speed   
& 2.3544 & \cellcolor{headerrose!55}2228.3518 
& \cellcolor{appgreen}2.4632 & \cellcolor{appyellow}4.4778 
& \textbf{2.4627} & \cellcolor{appteal}\textbf{4.1743} \\

\rowcolor{appblue}
hopper-medium-expert-v2 & G2\_balanced
& 1.6268 & \cellcolor{headerrose!55}2227.9148 
& \cellcolor{appgreen}1.6820 & \cellcolor{appyellow}4.4775 
& \textbf{1.6817} & \cellcolor{appteal}\textbf{4.1740} \\

\rowcolor{appteal}
hopper-medium-expert-v2 & G3\_efficient
& 0.3255 & \cellcolor{headerrose!55}2228.0630 
& \cellcolor{appgreen}0.3371 & \cellcolor{appyellow}4.4775 
& \textbf{0.3370} & \cellcolor{appteal}\textbf{4.1739} \\
\bottomrule
\end{tabular}%
}
\end{table*}

\begin{table*}[t]
\centering
\caption{Exact support-mismatch comparison using the contradictory and extrapolative goal cells underlying the conservatism summary. Each method entry reports return / KL / adaptation benefit.}
\label{tab:exact_support_mismatch_app}
\vspace{4pt}
\small
\setlength{\tabcolsep}{4pt}
\renewcommand{\arraystretch}{1.08}
\resizebox{\textwidth}{!}{%
\begin{tabular}{llllll}
\toprule
\rowcolor{headerpurple}
\textbf{Environment} & \textbf{Goal} & \textbf{Type} & \textbf{Prior Only} & \textbf{KL-Reg} & \textbf{PoE} \\
\midrule

\rowcolor{headerbeige!45}
halfcheetah-medium-v2    & G\_extreme\_speed & \cellcolor{appyellow}\textbf{extrapolative}
& 26.5687 / 1.0467 / 0.2435
& 26.5574 / 0.6857 / 0.2322
& \cellcolor{appgreen}\textbf{26.5574 / 0.6670 / 0.2322} \\

\rowcolor{headerrose!28}
halfcheetah-medium-v2    & G\_reverse       & \cellcolor{apppink}\textbf{contradictory}
& -5.1736 / 1.0007 / -0.0470
& -5.1423 / 2.9969 / -0.0157
& \cellcolor{appgreen}\textbf{-5.1380 / 0.0339 / -0.0114} \\

\rowcolor{headerbeige!30}
hopper-medium-v2         & G\_extreme\_speed & \cellcolor{appyellow}\textbf{extrapolative}
& 10.2390 / 121.7781 / -0.1657
& \cellcolor{apppurple}10.4248 / 5.2950 / 0.0201
& \cellcolor{appgreen}\textbf{10.4243 / 4.5908 / 0.0196} \\

\rowcolor{headerrose!18}
hopper-medium-v2         & G\_reverse       & \cellcolor{apppink}\textbf{contradictory}
& -2.0691 / 4.5208 / -0.0883
& -1.9957 / 2.3958 / -0.0149
& \cellcolor{appgreen}\textbf{-1.9863 / 0.5422 / -0.0055} \\

\rowcolor{headerbeige!45}
walker2d-medium-v2       & G\_extreme\_speed & \cellcolor{appyellow}\textbf{extrapolative}
& 13.9329 / 87.9328 / 1.6253
& \cellcolor{apppurple}14.1093 / 27.4165 / 1.8017
& \cellcolor{appgreen}\textbf{14.1062 / 26.1528 / 1.7986} \\

\rowcolor{headerrose!28}
walker2d-medium-v2       & G\_reverse       & \cellcolor{apppink}\textbf{contradictory}
& -2.6007 / 9.7198 / -0.2395
& -2.4367 / 4.8423 / -0.0755
& \cellcolor{appgreen}\textbf{-2.4138 / 1.1470 / -0.0526} \\

\rowcolor{headerbeige!30}
hopper-medium-expert-v2  & G\_extreme\_speed & \cellcolor{appyellow}\textbf{extrapolative}
& 10.7973 / 1560.7803 / -0.9782
& 11.0001 / 22.5256 / -0.7753
& \cellcolor{appgreen}\textbf{11.1155 / 14.9877 / -0.6600} \\

\rowcolor{headerrose!18}
hopper-medium-expert-v2  & G\_reverse       & \cellcolor{apppink}\textbf{contradictory}
& -2.1957 / 6.2772 / 0.0592
& \cellcolor{apppurple}-2.1993 / 4.1289 / 0.0557
& \cellcolor{appteal}\textbf{-2.1997 / 4.0655 / 0.0552} \\
\bottomrule
\end{tabular}%
}
\end{table*}

\paragraph{Appendix interpretation.}
The count summary and exact-value tables play complementary roles. The count table provides the reviewer-facing pattern at a glance: PoE is favorable in most main-comparison cells and especially strong under degraded priors and support mismatch. The exact tables then show that these counts are not driven by a single anomalous cell. \texttt{halfcheetah-medium-v2} is a clear strength case, \texttt{hopper-medium-v2} and \texttt{walker2d-medium-v2} are broadly favorable but tighter, and \texttt{hopper-medium-expert-v2} is the main boundary case where the frozen actor is already strong and KL-Reg can be slightly better in raw return.

\subsection{Detailed Single-Environment D4RL Results}
\label{app:d4rl_single_env_detail}

For completeness, we provide detailed single-environment results on
\texttt{halfcheetah-medium-v2} corresponding to the aggregate multi-environment
summaries reported in the main text. In particular, Table~\ref{tab:d4rl_prior_quality_detail}
expands the prior-degradation study by showing the per-goal behavior under trained,
noisy, and random deployment priors, together with the corresponding best PoE operating
point. Likewise, Table~\ref{tab:actor_quality_detail} expands the actor-quality
sensitivity study by reporting per-goal results for weak, medium, and strong frozen
actors, together with the corresponding PoE and KL-regularized improvements.

These detailed tables serve two purposes. First, they verify that the aggregate trends
in the main text are not driven by a single averaged configuration, but are consistently
reflected across the individual supported goals in the environment. Second, they make
the operating trade-offs more transparent by exposing the exact return, catastrophic
rate, KL deviation, and return-improvement values that are compressed in the aggregate
summaries. Together, Table~\ref{tab:d4rl_prior_quality_detail} and
Table~\ref{tab:actor_quality_detail} therefore provide the single-environment evidence
underlying the broader D4RL robustness and actor-quality conclusions discussed in the
main paper.

\begin{table*}[t]
\centering
\caption{Detailed single-environment robustness to deployment-prior degradation on
\texttt{halfcheetah-medium-v2}. We compare direct prior-only adaptation with the best PoE
operating point under degraded-prior settings. Higher return is better; lower catastrophic
rate and lower KL from the frozen actor are better.}
\label{tab:d4rl_prior_quality_detail}
\vspace{4pt}
\small
\setlength{\tabcolsep}{1.2pt}
\begin{tabular}{llccccccc}
\toprule
\rowcolor[RGB]{232,236,241}
\textbf{Goal} & \textbf{Prior Variant} &
\textbf{Prior Ret} & \textbf{Prior Cat} & \textbf{Prior KL} &
\cellcolor[RGB]{222,235,247}\textbf{Best PoE} &
\cellcolor[RGB]{222,235,247}\textbf{PoE Ret} &
\cellcolor[RGB]{222,235,247}\textbf{PoE Cat} &
\cellcolor[RGB]{222,235,247}\textbf{PoE KL} \\
\midrule
\rowcolor[RGB]{232,243,232}
G1: Speed
& Trained prior              & 5.4784 & 0.0015 & 1.9458   & $\alpha=0.1$ & 5.4727 & 0.0013 & 1.1534 \\
\rowcolor[RGB]{248,241,224}
& Noisy prior ($\sigma=0.05$)& 5.3051 & 0.0193 & 190.6071 & $\alpha=0.1$ & 5.4605 & 0.0039 & 42.4968 \\
\rowcolor[RGB]{248,223,227}
& Random prior               & 3.2040 & 0.8192 & 1047.4740 & $\alpha=0.3$ & 5.3772 & 0.0104 & 2.9930 \\
\midrule
\rowcolor[RGB]{232,243,232}
G2: Balanced
& Trained prior              & 3.3655 & 0.0021 & 1.9384   & $\alpha=0.1$ & 3.3626 & 0.0015 & 1.1539 \\
\rowcolor[RGB]{248,241,224}
& Noisy prior ($\sigma=0.05$)& 3.2125 & 0.0338 & 125.0987 & $\alpha=0.1$ & 3.3504 & 0.0046 & 35.2850 \\
\rowcolor[RGB]{248,223,227}
& Random prior               & 1.9986 & 0.7026 & 815.4703  & $\alpha=0.3$ & 3.2759 & 0.0175 & 2.7714 \\
\midrule
\rowcolor[RGB]{232,243,232}
G3: Efficiency
& Trained prior              & 1.0285 & 0.0077 & 1.8433   & $\alpha=0.1$ & 1.0279 & 0.0082 & 1.1643 \\
\rowcolor[RGB]{248,241,224}
& Noisy prior ($\sigma=0.05$)& 0.9168 & 0.0462 & 88.7405  & $\alpha=0.1$ & 1.0188 & 0.0111 & 27.5441 \\
\rowcolor[RGB]{248,223,227}
& Random prior               & 0.4285 & 0.6113 & 603.1209  & $\alpha=0.3$ & 0.9837 & 0.0210 & 2.5218 \\
\bottomrule
\end{tabular}
\end{table*}

The table shows that the same aggregate pattern appears at the per-goal level: noisy and undertrained priors are still manageable for PoE, while the random-prior setting is where prior-only adaptation fails most sharply.

\begin{table*}[t]
\centering
\caption{Detailed single-environment actor-quality sensitivity on
\texttt{halfcheetah-medium-v2}. Weak, medium, and strong actors correspond to
behavioral-cloning snapshots at epochs 5, 15, and 50, with validation NLL
$=-3.17/-3.91/-4.19$, respectively. The deployment prior is held fixed. $\Delta$
denotes return improvement over the corresponding frozen actor.}
\label{tab:actor_quality_detail}
\vspace{4pt}
\small
\setlength{\tabcolsep}{3pt}
\renewcommand{\arraystretch}{1.06}
\resizebox{\textwidth}{!}{%
\begin{tabular}{llccccccc}
\toprule
\rowcolor[RGB]{232,236,241}
\textbf{Actor} & \textbf{Val.\ NLL} & \textbf{Goal} & \textbf{Frozen Ret} & \textbf{PoE Ret} $(\alpha)$ & \textbf{KL-Reg Ret} $(\beta)$ & \textbf{PoE $\Delta$} & \textbf{KL-Reg $\Delta$} & \textbf{Cat\%$_{\mathrm{PoE}}$} \\
\midrule
\rowcolor[RGB]{248,248,250}
weak   & $-3.168$ & G1\_speed     & 5.4263 & 5.4651\,(0.1) & 5.4652\,(0.1) & +0.0388 & +0.0389 & 0.26 \\
\rowcolor[RGB]{248,248,250}
       &          & G2\_balanced  & 3.3322 & 3.3572\,(0.1) & 3.3571\,(0.1) & +0.0250 & +0.0249 & 0.30 \\
\rowcolor[RGB]{248,248,250}
       &          & G3\_efficient & 1.0046 & 1.0223\,(0.1) & 1.0223\,(0.1) & +0.0177 & +0.0177 & 1.13 \\
\midrule
\rowcolor[RGB]{242,247,247}
medium & $-3.914$ & G1\_speed     & 5.4187 & 5.4629\,(0.1) & 5.4633\,(0.1) & +0.0443 & +0.0446 & 0.24 \\
\rowcolor[RGB]{242,247,247}
       &          & G2\_balanced  & 3.3312 & 3.3562\,(0.1) & 3.3565\,(0.1) & +0.0250 & +0.0252 & 0.28 \\
\rowcolor[RGB]{242,247,247}
       &          & G3\_efficient & 1.0100 & 1.0224\,(0.1) & 1.0225\,(0.1) & +0.0124 & +0.0125 & 1.13 \\
\midrule
\rowcolor[RGB]{248,248,250}
strong & $-4.194$ & G1\_speed     & 5.4174 & 5.4641\,(0.1) & 5.4638\,(0.1) & +0.0466 & +0.0464 & 0.26 \\
\rowcolor[RGB]{248,248,250}
       &          & G2\_balanced  & 3.3315 & 3.3570\,(0.1) & 3.3569\,(0.1) & +0.0255 & +0.0254 & 0.30 \\
\rowcolor[RGB]{248,248,250}
       &          & G3\_efficient & 1.0119 & 1.0227\,(0.1) & 1.0228\,(0.1) & +0.0108 & +0.0109 & 1.11 \\
\bottomrule
\end{tabular}%
}
\end{table*}

The detailed actor-quality numbers confirm the aggregate story: the stronger anchors remain the best conservative starting points, and PoE stays slightly ahead of KL-Reg in return while keeping KL lower across all three goals.

\subsection{Additional Multi-Environment D4RL Comparison Tables}
\label{app:d4rl_multi_env_compact}

For completeness, we report two compact multi-environment D4RL comparison tables that
make the headline frozen-actor baselines explicit across all evaluated environments and
deployment goals. These tables complement the main-text D4RL figures and aggregate
robustness analyses by showing direct side-by-side comparisons among the frozen actor,
prior-only adaptation, additive mixing, KL-regularized post-hoc adaptation, and PoE.

Table~\ref{tab:d4rl_main_comparison_compact} provides the full compact numerical comparison.
For each environment-goal pair, it reports the goal-conditioned return together with the
KL divergence from the frozen actor, making it possible to inspect the return--conservatism
tradeoff directly for all principal frozen-actor baselines. This table is useful when the
reader wants the exact side-by-side values underlying the broader trends discussed in the
main text.

Table~\ref{tab:d4rl_main_comparison_summary} provides a higher-level reading aid. Instead
of reproducing the full interpretation from the raw numbers, it highlights, for each
environment-goal pair, which method achieves the strongest return and which adaptive method
stays closest to the frozen actor in KL. Read together, Tables~\ref{tab:d4rl_main_comparison_compact}
and~\ref{tab:d4rl_main_comparison_summary} make clear that the main empirical pattern is
not uniform dominance by a single method on every metric. Rather, they expose the central
tradeoff emphasized throughout the paper: prior-only adaptation can sometimes achieve the
largest raw return, while PoE and KL-regularized adaptation typically deliver a more
favorable balance between deployment performance and conservative deviation from the
validated actor.

\begin{table*}[t]
\centering
\caption{Compact multi-environment D4RL headline comparison. Each entry reports
goal-conditioned return / KL divergence from the frozen actor. The table makes the
closest frozen-actor baselines explicit across all evaluated environments and goals.}
\label{tab:d4rl_main_comparison_compact}
\vspace{4pt}
\small
\setlength{\tabcolsep}{3pt}
\renewcommand{\arraystretch}{1.08}
\resizebox{\textwidth}{!}{%
\begin{tabular}{llccccc}
\toprule
\rowcolor[RGB]{232,236,241}
\textbf{Env} & \textbf{Goal} & \textbf{Frozen} & \textbf{Prior Only} & \textbf{Additive} & \textbf{KL-Reg} & \textbf{PoE} \\
\midrule
\rowcolor[RGB]{248,248,250}
\texttt{halfcheetah-medium-v2} & G1\_speed     & 5.403 / 0.000 & 5.453 / 1.047 & 5.433 / 0.263 & 5.443 / 0.572 & 5.444 / 0.390 \\
\rowcolor[RGB]{248,248,250}
\texttt{halfcheetah-medium-v2} & G2\_balanced  & 3.325 / 0.000 & 3.352 / 1.043 & 3.340 / 0.262 & 3.346 / 0.575 & 3.347 / 0.391 \\
\rowcolor[RGB]{248,248,250}
\texttt{halfcheetah-medium-v2} & G3\_efficient & 1.011 / 0.000 & 1.023 / 1.005 & 1.016 / 0.251 & 1.021 / 0.618 & 1.021 / 0.400 \\
\midrule
\rowcolor[RGB]{242,247,247}
\texttt{hopper-medium-expert-v2} & G1\_speed     & 2.455 / 0.000 & 2.260 / 1560.719 & 2.311 / 392.886 & 2.282 / 27.656 & 2.273 / 29.407 \\
\rowcolor[RGB]{242,247,247}
\texttt{hopper-medium-expert-v2} & G2\_balanced  & 1.678 / 0.000 & 1.621 / 146.257  & 1.638 / 37.213  & 1.649 / 3.336  & 1.644 / 2.826 \\
\rowcolor[RGB]{242,247,247}
\texttt{hopper-medium-expert-v2} & G3\_efficient & 0.336 / 0.000 & 0.331 / 6.676    & 0.335 / 1.676   & 0.334 / 2.180  & 0.333 / 1.262 \\
\midrule
\rowcolor[RGB]{248,248,250}
\texttt{hopper-medium-v2} & G1\_speed     & 2.181 / 0.000 & 2.148 / 121.769 & 2.141 / 31.174 & 2.184 / 5.276 & 2.184 / 4.986 \\
\rowcolor[RGB]{248,248,250}
\texttt{hopper-medium-v2} & G2\_balanced  & 1.541 / 0.000 & 1.564 / 5.931   & 1.537 / 1.502  & 1.553 / 2.081 & 1.553 / 2.090 \\
\rowcolor[RGB]{248,248,250}
\texttt{hopper-medium-v2} & G3\_efficient & 0.309 / 0.000 & 0.318 / 4.524   & 0.310 / 1.123  & 0.314 / 2.143 & 0.315 / 2.099 \\
\midrule
\rowcolor[RGB]{242,247,247}
\texttt{walker2d-medium-v2} & G1\_speed     & 2.562 / 0.000 & 2.887 / 87.928 & 2.826 / 21.892 & 2.900 / 14.795 & 2.904 / 15.234 \\
\rowcolor[RGB]{242,247,247}
\texttt{walker2d-medium-v2} & G2\_balanced  & 1.732 / 0.000 & 1.863 / 10.604 & 1.784 / 2.655  & 1.836 / 5.774  & 1.838 / 5.720 \\
\rowcolor[RGB]{242,247,247}
\texttt{walker2d-medium-v2} & G3\_efficient & 0.349 / 0.000 & 0.374 / 9.688  & 0.358 / 2.427  & 0.368 / 5.367  & 0.368 / 5.286 \\
\bottomrule
\end{tabular}%
}
\end{table*}

\begin{table*}[t]
\centering
\caption{Compact multi-environment summary from the off-policy evaluator (diagnostic only). We highlight the strongest-return method and the lowest-reported-KL adaptive method per environment-goal pair. The ``Lowest adaptive KL'' column reflects each rule's stochastic-KL convention (PoE and KL-Reg differ here only by the factor $1+\beta$, not by their deterministic action); the main-body headline (Fig.~\ref{fig:headline_rollout}) uses rollout return directly.}
\label{tab:d4rl_main_comparison_summary}
\vspace{4pt}
\small
\setlength{\tabcolsep}{3pt}
\renewcommand{\arraystretch}{1.08}
\resizebox{\textwidth}{!}{%
\begin{tabular}{llcccccccc}
\toprule
\rowcolor[RGB]{232,236,241}
\textbf{Env} & \textbf{Goal} & \textbf{Frozen} & \textbf{Prior Only} & \textbf{Additive} & \textbf{KL-Reg} & \textbf{PoE} & \textbf{Best return} & \textbf{Lowest adaptive KL} \\
\midrule
\rowcolor[RGB]{248,248,250}
\texttt{halfcheetah-medium-v2} & G1\_speed     & 5.403 / 0.000 & 5.453 / 1.047 & 5.433 / 0.263 & 5.443 / 0.572 & 5.444 / 0.390 & Prior Only (5.453) & Additive (0.263) \\
\rowcolor[RGB]{248,248,250}
\texttt{halfcheetah-medium-v2} & G2\_balanced  & 3.325 / 0.000 & 3.352 / 1.043 & 3.340 / 0.262 & 3.346 / 0.575 & 3.347 / 0.391 & Prior Only (3.352) & Additive (0.262) \\
\rowcolor[RGB]{248,248,250}
\texttt{halfcheetah-medium-v2} & G3\_efficient & 1.011 / 0.000 & 1.023 / 1.005 & 1.016 / 0.251 & 1.021 / 0.618 & 1.021 / 0.400 & Prior Only (1.023) & Additive (0.251) \\
\midrule
\rowcolor[RGB]{242,247,247}
\texttt{hopper-medium-expert-v2} & G1\_speed     & 2.455 / 0.000 & 2.260 / 1560.719 & 2.311 / 392.886 & 2.282 / 27.656 & 2.273 / 29.407 & Frozen Actor (2.455) & KL-Reg (27.656) \\
\rowcolor[RGB]{242,247,247}
\texttt{hopper-medium-expert-v2} & G2\_balanced  & 1.678 / 0.000 & 1.621 / 146.257  & 1.638 / 37.213  & 1.649 / 3.336  & 1.644 / 2.826 & Frozen Actor (1.678) & PoE (2.826) \\
\rowcolor[RGB]{242,247,247}
\texttt{hopper-medium-expert-v2} & G3\_efficient & 0.336 / 0.000 & 0.331 / 6.676    & 0.335 / 1.676   & 0.334 / 2.180  & 0.333 / 1.262 & Frozen Actor (0.336) & PoE (1.262) \\
\midrule
\rowcolor[RGB]{248,248,250}
\texttt{hopper-medium-v2} & G1\_speed     & 2.181 / 0.000 & 2.148 / 121.769 & 2.141 / 31.174 & 2.184 / 5.276 & 2.184 / 4.986 & PoE (2.184) & PoE (4.986) \\
\rowcolor[RGB]{248,248,250}
\texttt{hopper-medium-v2} & G2\_balanced  & 1.541 / 0.000 & 1.564 / 5.931   & 1.537 / 1.502  & 1.553 / 2.081 & 1.553 / 2.090 & Prior Only (1.564) & Additive (1.502) \\
\rowcolor[RGB]{248,248,250}
\texttt{hopper-medium-v2} & G3\_efficient & 0.309 / 0.000 & 0.318 / 4.524   & 0.310 / 1.123  & 0.314 / 2.143 & 0.315 / 2.099 & Prior Only (0.318) & Additive (1.123) \\
\midrule
\rowcolor[RGB]{242,247,247}
\texttt{walker2d-medium-v2} & G1\_speed     & 2.562 / 0.000 & 2.887 / 87.928 & 2.826 / 21.892 & 2.900 / 14.795 & 2.904 / 15.234 & PoE (2.904) & KL-Reg (14.795) \\
\rowcolor[RGB]{242,247,247}
\texttt{walker2d-medium-v2} & G2\_balanced  & 1.732 / 0.000 & 1.863 / 10.604 & 1.784 / 2.655  & 1.836 / 5.774  & 1.838 / 5.720 & Prior Only (1.863) & Additive (2.655) \\
\rowcolor[RGB]{242,247,247}
\texttt{walker2d-medium-v2} & G3\_efficient & 0.349 / 0.000 & 0.374 / 9.688  & 0.358 / 2.427  & 0.368 / 5.367  & 0.368 / 5.286 & Prior Only (0.374) & Additive (2.427) \\
\bottomrule
\end{tabular}%
}
\end{table*}

Taken together, Tables~\ref{tab:d4rl_main_comparison_compact} and~\ref{tab:d4rl_main_comparison_summary} show that the strong halfcheetah pattern is not a single-environment artifact: PoE remains the main conservative adaptive choice in the environments where frozen-actor steering is genuinely useful, while the stronger frozen anchors can still dominate in the harder boundary case.

\subsection{Extended Ablation Studies}
\label{app:extended-ablation}

We next examine within-family sensitivity to prior construction and refinement strength.
Figure~\ref{fig:appendix_ablation_heatmap} summarizes critic-prior temperature sweeps
($\tau \in \{0.1,1.0,5.0\}$) and refinement-prior $\alpha$ sweeps
($\alpha \in \{0.3,0.5,0.7\}$) using the released appendix CSVs.

Three qualitative patterns emerge. First, critic-prior temperature has only a
mild effect on return but a clearer effect on lower-tail behavior, with the
sharpest prior giving the best risk profile within the critic-prior family.
Second, within the refinement-prior family, changing $\alpha$ produces the
expected adaptation--conservatism tradeoff: larger actor weight improves proximity
to the pretrained actor and slightly improves lower-tail behavior. Third, the
refinement-prior family is more tuning-robust than the critic-prior family.

The same conclusion holds on the original metric scales: the actor-only policy
remains the strongest conservative anchor, while the refinement-prior variants
remain much closer to that anchor than critic-only or evidence-free ablations.

\begin{figure}[t]
   \centering
   \includegraphics[width=\linewidth]{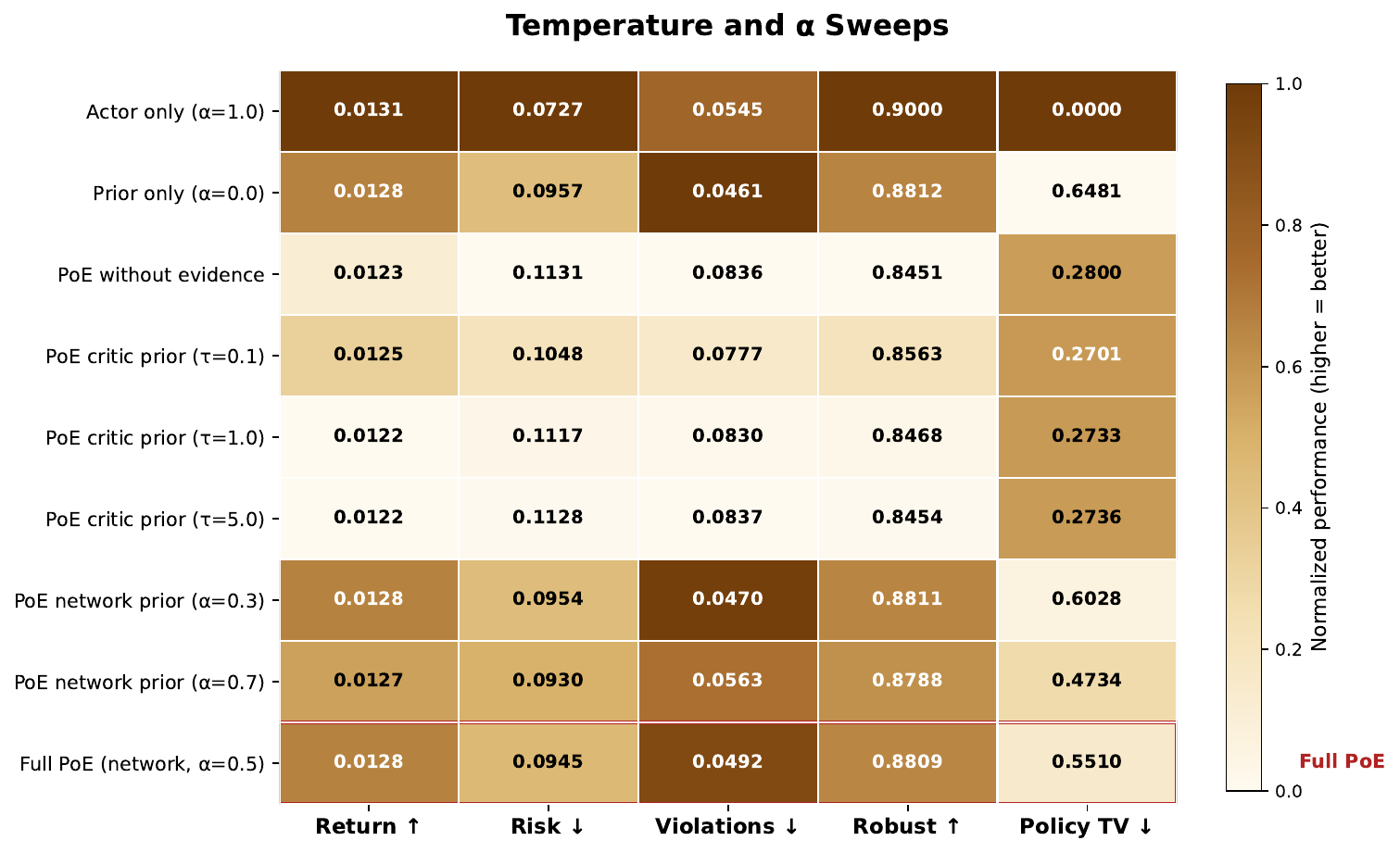}
   \caption{Extended ablation across critic-prior temperature sweeps and refinement-prior $\alpha$ sweeps.}
   \label{fig:appendix_ablation_heatmap}
\end{figure}

\subsection{Policy Stability under Critic Degradation}
\label{app:critic-degradation}

We next examine how critic-dependent policies behave when the critic used for action
selection is corrupted by Gaussian noise
$\sigma\in\{0.01,0.05,0.10,0.25,0.50,1.00\}$. This experiment is a failure-tolerance diagnostic rather than a primary performance benchmark. Evaluation uses the goal-weighted dataset reward $r_g(s,a)=g\cdot rc(s,a)$ rather than critic values, so the corrupted critic affects action selection but not scoring.

We track two metrics as the corruption noise grows: the total variation distance
from the frozen actor, $\mathrm{TV}(\pi,\pi_\theta)$, and the support-violation rate.
Critic-Greedy saturates both metrics almost immediately. PoE with the refinement prior remains close to the frozen actor's behavioral envelope throughout the sweep. PoE with the critic prior is less stable, reflecting direct propagation of critic corruption into the prior itself. We omit the figure from this version of the paper; the sweep values are in the released CSVs.


\subsection{Support Preservation Diagnostic}
\label{app:support-preservation}

Proposition~\ref{prop:interp}(iii) establishes exact support preservation:
$\pi_\theta(a\mid s)=0 \Rightarrow \pi_{\mathrm{ref}}(a\mid s,g)=0$ for all
$\alpha>0$. Because learned actors often assign very small rather than exactly zero
probability to many actions, we also report an approximate support-violation rate.
The results appear in Table~\ref{tab:support_preservation}.

Table~\ref{tab:support_preservation} reports the fraction of selected actions
satisfying $\pi_\theta(a\mid s)<\epsilon$ for
$\epsilon\in\{0.001,0.01,0.05\}$. PoE with the refinement prior stays below the frozen
actor’s own violation rate at all three thresholds, while Critic-Greedy violates
support almost everywhere. The critic-prior variant lies between these extremes,
reflecting its partial dependence on critic-derived scores.

\begin{table}[tb]
\centering
\small
\caption{Approximate support-violation rates at three thresholds $\epsilon$. PoE
(Refinement Prior) remains below the frozen actor’s own violation rate, while
Critic-Greedy violates support almost everywhere.}
\label{tab:support_preservation}
\vspace{4pt}
\begin{tabular}{lccc}
\toprule
\rowcolor[RGB]{191,223,223}
\textbf{Method} & \textbf{$\epsilon=0.001$} & \textbf{$\epsilon=0.01$} & \textbf{$\epsilon=0.05$} \\
\midrule
\rowcolor[RGB]{242,247,247}
Pretrained Actor       & 0.0019 & 0.0238 & 0.1184 \\
\rowcolor[RGB]{248,223,227}
Critic-Greedy          & 0.9986 & 0.9986 & 0.9986 \\
\rowcolor[RGB]{223,240,216}
PoE (Refinement Prior) & \textbf{0.0012} & \textbf{0.0188} & \textbf{0.0885} \\
\rowcolor[RGB]{242,247,247}
PoE (Critic Prior)     & 0.0476 & 0.1536 & 0.3484 \\
\bottomrule
\end{tabular}
\end{table}

\subsection{Q-value Analysis of Selected Actions}
\label{app:qvalue-diagnostic}

We also report a diagnostic analysis of the original uncorrupted deployment critic.
Table~\ref{tab:qvalue_actions} summarizes the mean, standard deviation, minimum, and
maximum Q-values of actions selected by each policy under that critic.

Table~\ref{tab:qvalue_actions} shows that Critic-Greedy selects a tight high-value
action set, consistent with direct critic maximization. By contrast, PoE with the
refinement prior selects actions whose Q-value distribution is nearly identical to
dataset actions and to the frozen actor, which is exactly the behavior expected from an
actor-preserving refinement mechanism. The critic-prior variant exhibits a wider range,
again reflecting greater sensitivity to critic-derived structure.

\begin{table}[tb]
\centering
\small
\caption{Q-value statistics of selected actions under the original uncorrupted critic.
PoE (Refinement Prior) stays close to dataset and frozen-actor behavior, while
Critic-Greedy reflects aggressive critic maximization.}
\label{tab:qvalue_actions}
\vspace{4pt}
\begin{tabular}{lcccc}
\toprule
\rowcolor[RGB]{221,214,243}
\textbf{Method} & \textbf{Mean} & \textbf{Std.} & \textbf{Min} & \textbf{Max} \\
\midrule
\rowcolor[RGB]{248,248,252}
Dataset actions        & 0.0127 & 0.0074 & -0.0077 & 0.0305 \\
\rowcolor[RGB]{244,244,249}
Behavior Cloning       & 0.0150 & 0.0070 &  0.0024 & 0.0461 \\
\rowcolor[RGB]{248,248,252}
Pretrained Actor       & 0.0131 & 0.0080 & -0.0109 & 0.0461 \\
\rowcolor[RGB]{245,240,230}
BCQ                    & 0.0186 & 0.0036 &  0.0160 & 0.0461 \\
\rowcolor[RGB]{248,223,227}
Critic-Greedy          & \textbf{0.0330} & \textbf{0.0021} & \textbf{0.0297} & \textbf{0.0566} \\
\rowcolor[RGB]{223,235,247}
PoE (Refinement Prior) & 0.0128 & 0.0074 & -0.0080 & 0.0297 \\
\rowcolor[RGB]{242,247,247}
PoE (Critic Prior)     & 0.0122 & 0.0084 & -0.0109 & 0.0458 \\
\bottomrule
\end{tabular}
\end{table}

\subsection{Additional Benchmark Analysis on Frozen-Actor Goal Adaptation}
\label{app:d4rl_goal_adaptation}

We provide an additional single-environment benchmark analysis on
\texttt{halfcheetah-medium-v2} to study deployment-time goal adaptation under a frozen
actor. This is a supplement to the main D4RL results in Section~\ref{sec:results},
where the main-paper comparison focuses on matched-KL tradeoffs and prior degradation.

Table~\ref{tab:d4rl_goal_adaptation_appendix} reports return, catastrophic rate,
robustness, and KL divergence from the frozen actor under three deployment goals.
The table shows that the frozen actor is not uniformly optimal once the deployment goal
changes, while prior-only adaptation attains the highest return at substantially larger
KL divergence. PoE provides a smoother adaptation--conservatism tradeoff across
$\alpha$.

\begin{table*}[t]
\centering
\caption{Additional benchmark analysis on \texttt{halfcheetah-medium-v2} under three
deployment goals. Higher return and robustness are better; lower catastrophic rate and
KL are better.}
\label{tab:d4rl_goal_adaptation_appendix}
\vspace{4pt}
\small
\setlength{\tabcolsep}{4pt}
\resizebox{\textwidth}{!}{%
\begin{tabular}{lcccccccccccc}
\toprule
& \multicolumn{4}{c}{\cellcolor[RGB]{245,232,210}\textbf{G1: Speed}}
& \multicolumn{4}{c}{\cellcolor[RGB]{226,239,226}\textbf{G2: Balanced}}
& \multicolumn{4}{c}{\cellcolor[RGB]{228,232,246}\textbf{G3: Efficiency}} \\
\cmidrule(lr){2-5} \cmidrule(lr){6-9} \cmidrule(lr){10-13}
\rowcolor[RGB]{238,240,242}
\textbf{Method}
& \textbf{Return $\uparrow$} & \textbf{Cat.\ $\downarrow$} & \textbf{Rob $\uparrow$} & \textbf{KL $\downarrow$}
& \textbf{Return $\uparrow$} & \textbf{Cat.\ $\downarrow$} & \textbf{Rob $\uparrow$} & \textbf{KL $\downarrow$}
& \textbf{Return $\uparrow$} & \textbf{Cat.\ $\downarrow$} & \textbf{Rob $\uparrow$} & \textbf{KL $\downarrow$} \\
\midrule
\rowcolor[RGB]{250,250,250}
Frozen Actor
& 5.4012 & 0.0079 & 0.9921 & 0.0000
& 3.3236 & 0.0094 & 0.9906 & 0.0000
& 1.0107 & 0.0226 & 0.9774 & 0.0000 \\
\rowcolor[RGB]{245,236,240}
Prior Only
& \textbf{5.4784} & 0.0015 & 0.9985 & 1.9458
& \textbf{3.3655} & 0.0021 & 0.9979 & 1.9384
& \textbf{1.0285} & \textbf{0.0077} & \textbf{0.9923} & 1.8433 \\
\rowcolor[RGB]{232,244,234}
PoE ($\alpha = 0.3$)
& 5.4634 & \textbf{0.0011} & \textbf{0.9989} & 0.6666
& 3.3582 & \textbf{0.0010} & \textbf{0.9990} & 0.6691
& 1.0263 & 0.0081 & 0.9919 & 0.7078 \\
\rowcolor[RGB]{247,247,247}
PoE ($\alpha = 0.5$)
& 5.4541 & 0.0012 & 0.9988 & 0.3970
& 3.3529 & 0.0018 & 0.9982 & 0.3994
& 1.0241 & 0.0089 & 0.9911 & 0.4348 \\
\rowcolor[RGB]{250,250,250}
PoE ($\alpha = 0.7$)
& 5.4407 & 0.0017 & 0.9983 & 0.2084
& 3.3455 & 0.0020 & 0.9980 & 0.2099
& 1.0207 & 0.0103 & 0.9897 & 0.2321 \\
\bottomrule
\end{tabular}%
}
\end{table*}

\section{Implementation Details and Reproducibility}
\label{app:implementation-details}

\subsection{Training and Evaluation Setup}
\label{app:training-setup}

All models were implemented in PyTorch. The actor is trained once on the offline
dataset via behavioral cloning and then frozen; the refinement prior is trained
separately on the same dataset via contrastive goal-weighted behavioral cloning.
The Product-of-Experts layer introduces no learnable parameters at deployment time,
it only reweights actor and prior distributions according to $\alpha$ via
Eq.~\eqref{eq:poe-policy}. A held-out goal-conditioned critic is trained with a
different random seed and used only as a within-protocol diagnostic scorer in
Appendix~\ref{app:evaluation_protocol}; the main body's returns are from
real MuJoCo rollouts (Sec.~\ref{sec:setup}). Unless otherwise specified,
results are averaged over the five evaluation seeds reported in the main body.

\subsection{Inference-Time Procedure}
\label{app:algorithm}

Algorithm~\ref{alg:poe_refinement} gives the complete deployment-time refinement 
procedure. The key property is that adaptation is entirely parameter-free at 
deployment: no gradients are computed, no model weights are updated, and no new 
data are accessed. The procedure requires only a forward pass through the frozen 
actor and the offline-trained prior, followed by a pointwise reweighting over 
$|\mathcal{A}|$ actions. The computational cost is therefore 
$\mathcal{O}(|\mathcal{A}|)$ per decision step, making it practical even for 
large action spaces.

Changing the deployment objective requires only updating $g$; the actor, prior, 
and $\alpha$ are all fixed. This means the same trained components can serve 
multiple deployment contexts simultaneously, with zero additional training cost 
per new goal. For the continuous-action D4RL setting, the same procedure applies 
with Gaussian distributions: the PoE composition remains Gaussian in closed form 
(see Section~\ref{sec:continuous-actions}), so no discrete enumeration is needed.

\begin{algorithm}[t]
  \caption{Deployment-Time Product-of-Experts Policy Refinement}
  \label{alg:poe_refinement}
  \begin{algorithmic}[1]
    \Require Frozen actor $\pi_\theta(a\mid s)$; prior $\rho_\phi(a\mid s,g)$;
             coefficient $\alpha\in[0,1]$; state $s$; goal $g$
    \Ensure $\pi_{\mathrm{ref}}(\cdot\mid s,g)$ and action $a^\star$
    \For{each $a\in\mathcal{A}$}
      \State $w(a)\gets\pi_\theta(a\mid s)^{\alpha}\,\rho_\phi(a\mid s,g)^{1-\alpha}$
    \EndFor
    \State $Z\gets\sum_{a'\in\mathcal{A}}w(a')$
    \For{each $a\in\mathcal{A}$}
      \State $\pi_{\mathrm{ref}}(a\mid s,g)\gets w(a)/Z$
    \EndFor
    \State $a^\star\sim\pi_{\mathrm{ref}}(\cdot\mid s,g)$
    \State\Return $\pi_{\mathrm{ref}}(\cdot\mid s,g),\;a^\star$
  \end{algorithmic}
\end{algorithm}

\section{AI Usage Statement}
\label{app:ai-statement}

Generative AI tools were used in a limited and supportive capacity during the 
preparation of this manuscript. ChatGPT was employed for writing assistance, 
including refinement of technical language, improvement of clarity and organization, 
and correction of grammar and spelling. Claude was used to assist with code 
refactoring, formatting, and software-structure improvements. All scientific 
content, experimental design, results, interpretations, and conclusions were 
conceived, implemented, and validated by the authors. No data, results, or analyses 
were generated or fabricated by AI systems. All AI-assisted outputs were reviewed, 
edited, and verified by the authors to ensure accuracy, originality, and consistency 
with scientific standards.

\end{document}